\documentclass{article} % For LaTeX2e
\usepackage{iclr2024_conference,times}

\title{Horizon-Free Regret for \\ Linear Markov Decision Processes}

\iclrfinalcopy

\author{Zihan Zhang\thanks{ Department of Electrical and Computer Engineering, Princeton University; email: \texttt{\{zz5478,jasonlee\}@princeton.edu}.}\And Jason D.~Lee\footnotemark[1]\And Yuxin Chen\thanks{Department of Statistics and Data Science, University of Pennsylvania; email: \texttt{yuxinc@wharton.upenn.edu}.}\And 
 Simon S.~Du\thanks{Paul G. Allen School of Computer Science and Engineering, University of Washington; email: \texttt{ssdu@cs.washington.edu}.}
}

\usepackage[utf8]{inputenc} % allow utf-8 input
\usepackage[T1]{fontenc}    % use 8-bit T1 fonts
\usepackage{hyperref}       % hyperlinks
\usepackage{url}            % simple URL typesetting
\usepackage{booktabs}       % professional-quality tables
\usepackage{amsfonts}       % blackboard math symbols
\usepackage{nicefrac}       % compact symbols for 1/2, etc.
\usepackage{microtype}      % microtypography
\usepackage{xcolor}         % colors
\usepackage{amsmath}
\usepackage{amsthm}

\usepackage{mathtools}
\usepackage{graphicx}
\usepackage{amssymb}
\usepackage{bbm}
\usepackage{xpatch}
\usepackage{mathrsfs}

\usepackage{url}
\usepackage{array}
\usepackage{wrapfig}
\usepackage[normalem]{ulem} % to use \sout
\usepackage{enumerate}
\usepackage{enumitem}
\usepackage{footnote}

\newtheorem{lemma}{Lemma}
\newtheorem{theorem}{Theorem}

\newtheorem{assumption}{Assumption}

\newtheorem{example}{Example}

\def\vkh{$V_h^k$}

\def\vkh1{$V_{h+1}^k$}
\def\vkh1'{$V_{h+1}^k(s')$}

\def\ovh{$V_{h}^*(\cdot)$}

\def\ovh1{$V_{h+1}^*$}
\def\ovh1'{$V_{h+1}^*(s')$}

\def\tvkh{$ \tilde{V}_h^k $}

\def\tvkh1{ $\tilde{V}_{h+1}^k$ }
\def\tvkh1'{$\tilde{V}_{h+1}^k(s')$}

\def\pkhs'{$P_{s_h^k,a_h^k,s'} $}

\def\hpksa{$\hat{P}^k_{s,a}$}
\def\hpksa'{$\hat{P}^k_{s,a,s'}$}
\def\hpkh{$\hat{P}^k_{s_h^k,a_h^k}$}
\def\hpkh'{$ \hat{P}^k_{s_h^k,a_h^k,s'} $}

\newcommand{\abs}[1]{\left|#1\right|}

\usepackage{algorithm}% http://ctan.org/pkg/algorithms
\usepackage{algorithmic}
\usepackage{pifont}

\usepackage{hhline}
\usepackage{makecell}

\newtheorem{claim}{Claim}

\usepackage{times}
\usepackage{xcolor}
\usepackage{array}
\usepackage{hyperref}
\usepackage{url}

\usepackage{multirow}

\usepackage{color}
\definecolor{yxc}{RGB}{255,0,0}

\newcommand{\mymid}{\,|\,}

\begin{document}

\maketitle

\begin{abstract}

A recent line of works showed regret bounds in reinforcement learning (RL) can be (nearly) independent of planning horizon, a.k.a.~the horizon-free bounds.
However, these regret bounds only apply to settings where a polynomial dependency on the size of transition model is allowed, such as tabular Markov Decision Process (MDP) and linear mixture MDP.
We give the first horizon-free bound for the popular linear MDP setting where the size of the transition model can be exponentially large or even uncountable. 
In contrast to prior works which explicitly estimate the transition model and compute the inhomogeneous value functions at different time steps, we directly estimate the value functions and confidence sets. 
We obtain the horizon-free bound by: (1) maintaining \emph{multiple} weighted least square estimators for the value functions; and (2) a structural lemma which shows the maximal total variation of the \emph{inhomogeneous} value functions is bounded by a polynomial factor of the feature dimension.

% Horizon-free episodic reinforcement learning (RL) has been widely studied in recent years. 
%In the case of tabular MDP and linear mixture MDP, researchers show that there exists computationally efficient algorithms to achieve horizon-free regret bounds. However, it is still unknown whether episodic linear MDP could be efficiently learned without polynomial dependence on the horizon.

%In this paper, we consider to learn a linear MDP with dimension  $d$, horizon  $H$, confidence parameter $\delta$, and total number of episodes $K$ , and propose a sample-efficient algorithm to achieve regret bound of $O(d^{5.5}\sqrt{K}+d^{6.5})\cdot \mathrm{polylog}(d,K,H,1/\delta)$. This is the first sub-linear regret bound with only \emph{logarithmic} dependence on $H$ for linear MDP.

%In technique, our contribution are two-fold: (1) we show that the maximal total variation of the \emph{inhomogeneous} value functions 
%$\sum_{h=1}^H \|V^*_h-V^*_{h+1}\|_{\infty}$ i
%can be bounded by $\mathrm{poly}(d)$ where $d$ is the dimension of the feature; (2) we maintain multiple weighted least square estimators for the \emph{inhomogeneous} value functions, and manage to bound the maximal total variance with respect to these value functions. 

%Final regret bound $\tilde{O}(d^{15/4}\sqrt{K})$.

\end{abstract}

%\begin{keywords}%
 % Horizon-free reinforcement learning; online learning; linear function approximation
%\end{keywords}

\section{Introduction}

% rl

In reinforcement learning (RL), an agent learns to interact with an unknown environment by observing the current states and taking actions sequentially. The goal of the agent is to maximize the accumulative reward. In RL, the sample complexity describes the number of samples needed to learn a near-optimal policy.

% function approximation, % gap
It has been shown that the sample complexity needs to scale with the state-action space for the tabular Markov Decision Process (MDP)~\citep{domingues2021episodic}.
When the state-action space is large, function approximation is often used to generalize across states and actions. One popular model is linear MDP where the transition model is assumed to be low-rank~\citep{yang2019sample,jin2019provably}.
We denote by $d$ the rank (or feature dimension) and the sample complexity will depend on $d$ instead of the size of the state-action space.
 
However, there remains a gap between our theoretical understanding of tabular MDPs and linear MDPs. For tabular MDPs, a line of works give horizon-free bounds, i.e., the sample complexity can be (nearly) independent of the planning horizon~\citep{wang2020long,zhang2020reinforcement}. More recently, the horizon-free bounds were also obtained in linear mixture MDP where the underlying MDP could be presented by a linear combination of $d$ known MDPs~\citep{ayoub2020model,modi2020sample,jia2020model}.\textbf{}
A natural theoretical question is:
\begin{center}
\textbf{\emph{Can we obtain horizon-free regret bounds for linear MDPs?}}
\end{center}

% why linear MDP is harder: 
One major technical challenge in obtaining horizon-free bounds is that the value function is inhomogeneous, i.e., for an MDP with planning horizon $H$, the optimal value functions for each time step $\{V^*_h\}_{h=1}^H$ can vary across $h=1,\ldots,H$.
For tabular MDPs and linear mixture MDPs, one can resolve this challenge by first estimating the transition kernel and then  computing the value function based on the learned model.
The sample complexity will then scale with the size of the transition kernel, which is homogeneous and does not depend on $H$.
For tabular MDPs and linear mixture MDPs, the dependence on the model size is allowed.\footnote{For linear mixture MDP, the model size scales linearly with the feature dimension.}
Unfortunately, this approach cannot be readily applied to handle linear MDPs because the model size of a linear MDP scales with the size of the state space, which can be exponentially large or even uncountable.

\paragraph{Contributions.}  In this paper, we answer the above question affirmatively by establishing a regret bound of  $\widetilde{O}(\mathrm{poly}(d)\sqrt{K})$ for linear MDPs. Formally we have the result below.
\begin{theorem}\label{thm:main} Choose $\mathtt{Reward-Confidence}$ as VOFUL (see Algorithm~\ref{alg:voful}). For any MDP satisfying the total-bounded reward assumption (Assumption~\ref{assumr}) and  linear MDP assumption (see Assumption~\ref{assuml}), then
with probability $1-\delta$, the regret of Algorithm~\ref{alg:main} is bounded by $\widetilde{O}(d^{5.5}\sqrt{K}+ d^{6.5})$, where logarithmic factors of $(d, K, H,1/\delta)$ are hidden by the $\widetilde{O}(\cdot )$ parameter.
\end{theorem}

%Although the propose algorithm (Algorithm~\ref{alg:main}) is inefficient in computation,
 By virtue of Theorem~\ref{thm:main}, we show that linear MDP has a sample-complexity with only poly-logarithmic dependence on $H$. Although the proposed algorithm (Algorithm~\ref{alg:main}) is inefficient in computation, we believe our method provides intuitions to remove the horizon dependence in the view of statistical efficiency. 
 % which implies that the linear MDP problem is essentially no harder than the contextual linear bandit problem in the view of statistical efficiency. 
 % This is consistent with the intuition that: \emph{the planning horizon does not affect the complexity of the underlying model}. 

In terms of technical innovations, we design a novel method to share the samples to solve different linear bandit problems. Since the optimal value functions are not homogeneous, we need to
 learn $H$ different linear bandit problems using the same dataset. We first show that it suffices to bound the regret for each single bandit problem, and then bound the maximal total variance by bounding the variation of the optimal value function.
%Using the weighted least-square estimator \cite{zhou2022computationally}, we 
See Section~\ref{sec:tec} for more details. Due to space limitation, we postpone the full proof  of Theorem~\ref{thm:main} to Appendix~\ref{app:reg}.

\subsection{Related Works}

\paragraph{Tabular MDPs.} There has been a long list of algorithms proposed for episodic tabular MDPs (e.g., \cite{kearns2002near,brafman2002r,kakade2003sample, agralwal2017optimistic,azar2017minimax,jin2018q,zhang2020almost,wang2020long,jin2020reward,zhang2020reinforcement,li2021breaking,li2021settling,li2023minimax,zhang2022horizon}). For finite-horizon {\em inhomogeneous} MDPs with the immediate reward at each step bounded by $1/H$, \cite{azar2017minimax,zhang2020almost,zanette2019tighter,li2021breaking} achieved asymptotically minimax-optimal regret $\widetilde{\Theta}(\sqrt{SAHK})$ (ignoring lower order terms), where $SA$ is the size of state-action space, $H$ is the planning horizon, and $K$ is the number of episodes.   
Motivated by a conjecture raised by \cite{jiang2018open}, \cite{wang2020long} developed---for {\em time-homogeneous} MDPs with total rewards in any episode bounded above by 1---the first sample complexity upper bound  that exhibits only logarithmic dependence on the horizon $H$, 
which was later on improved by \cite{zhang2020reinforcement} to yield a near-optimal regret bound of $\widetilde{O}(\sqrt{SAK}+S^2A)$. 
Subsequently, \cite{li2021settling,zhang2022horizon} proved that even the poly-logarithmic horizon dependency in the sample complexity can be removed, albeit at the price of suboptimal scaling with $SA$.

\paragraph{Contextual linear bandits.} The linear bandit problem has been extensively studied in past decades \citep{auer2002using,dani2008stochastic,chu2011contextual,abbasi2011improved}. For linear bandits with infinite arms,  the minimax regret bound of $\widetilde{\Theta}(\sqrt{dK})$ is achieved by OFUL \citep{abbasi2011improved}, where $d$ is the feature dimension and $K$ the number of rounds.  
With regards to variance-aware algorithms, \cite{zhang2021variance,kim2022improved} proposed VOFUL (VOFUL+) to achieve a variance-dependent regret bound of $\widetilde{O}\big(\mathrm{poly}(d)\sqrt{\sum_{k=1}^K \sigma^2_k}+\mathrm{poly}(d)\big)$, 
in the absence of the knowledge of $\{\sigma_k\}_{k=1}^K$ (with $\sigma_k$  the conditional variance  of the noise in the $k$-th round). 
By assuming prior knowledge of $\{\sigma_{k}\}_{k=1}^K$, \cite{zhou2021nearly,zhou2022computationally} obtained similar regret bounds with improved dependence on $d$. Another work \citep{faury2020improved} proposed a Bernstein-style confidence set for
the logistic bandit problem, also assuming availability of the noise variances.

\paragraph{RL with linear function approximation.}
It has been an important problem in the RL community to determine the generalization capability of linear function approximation \citep{jiang2017contextual, dann2018oracle,yang2019sample,jin2019provably,wang2019optimism,sun2019model,zanette2020learning,weisz2020exponential,li2021sample,ayoub2020model,zhang2021variance,kim2022improved,zhou2021nearly,zhou2022computationally,he2022nearly}. Several model assumptions have been proposed and exploited to capture the underlying dynamics via linear functions. For example, \cite{jiang2017contextual} investigated low Bellman-rank, which described the algebraic dimension between the decision process and value-function approximator. 
Another setting proposed and studied by 
\cite{jia2020model,ayoub2020model,modi2020sample} is that of linear mixture MDPs, 
which postulates that the underlying dynamics is a linear combination of $d$ known environments. 
Focusing on linear mixture MDPs, 
\cite{zhang2021variance} proposed the first sample-efficient algorithm to achieve horizon-free $\widetilde{O}(\mathrm{poly}(d)\sqrt{K})$ regret, and later on \cite{kim2022improved} obtained better $d$-dependency in the regret bound;  
further, a recent study~\cite{zhou2022computationally} designed a variance- \& uncertainty-aware exploration bonus with weighted least-square regression, achieving near-optimal regret bounds with computation efficiency.
Another recent strand of research 
\cite{yang2019sample,jin2019provably,he2022nearly,agarwal2022vo} studied the setting of linear MDPs, where the transition kernel and reward function are assumed to be linear functions of several known low-dimensional feature vectors.
Take episodic inhomogeneous linear MDPs for example: when the feature dimension is $d$ and the immediate reward in each step is bounded above by $1/H$, 
\citep{jin2019provably} established the regret bound of $\widetilde{O}(\sqrt{d^3H^2K})$, 
whereas the follow-up works \citet{he2022nearly, agarwal2022vo} improved the regret to $\widetilde{O}(d\sqrt{HK})$. 
It remained unclear whether and when horizon-free solutions are plausible in linear MDPs, 
in the hope of accommodating scenarios with exceedingly large $H$.

\section{Preliminaries}

In this section, we present the basics of MDPs and the learning process, and introduce our key assumptions. 
Throughout the paper, $\Delta(X)$ denotes the set of probability distributions over the set $X$.

\paragraph{Episodic MDPs.}  A finite-horizon episodic MDP can be represented by  
a  tuple $(\mathcal{S}, \mathcal{A}, R, P, K, H)$, where $\mathcal{S}$ denotes the state space containing $S$ states, $\mathcal{A}$ is the action space containing $A$ different actions, 
$R: \mathcal{S}\times\mathcal{A}\rightarrow \Delta([0,1])$ indicates the reward distribution,  
$P:  \mathcal{S}\times\mathcal{A}\rightarrow \Delta(\mathcal{S})$ represents the probability transition kernel,  
 $K$ stands for the total number of sample episodes that can be collected, and $H$ is the planning horizon.  
In particular,  $P$ is assumed throughout to be {\em time-homogeneous}, which is necessary to enable nearly horizon-free regret bounds; 
in light of this assumption, we denote by  
$P_{s,a}\coloneqq P(\cdot \mymid s, a)\in \Delta (\mathcal{S})$ the transition probability from state $s$ to state $s'$ while taking  action $a$.   The reward distribution $R$ is also assumed to be time-homogeneous, so that the immediate reweard at a state-action pair $(s,a)$ at any step $h$ is drawn from 
$R(s,a)$ with mean $\mathbb{E}_{r'\sim R(s,a)}[r']=r(s,a)$. 
% Without the loss of generality, we assume a fixed initial state $s_1$. 
%
Moreover, a deterministic and possibly non-stationary policy $\pi = \{\pi_h : \mathcal{S} \to \mathcal{A}\}_{h=1}^H$ describes an action selection strategy, 
with $\pi_h(s)$ specifying the action chosen in state $s$ at step $h$.

At each sample episode, the learner starts from an initial state $s_1$; 
for each step $h = 1,\ldots, H$, the learner observes the current state $s_h$, 
takes action $a_h$ accordingly, receives an immediate reward $r_h\sim R(s_h,a_h)$, and then the environment transits to the next state $s_{h+1}$ in accordance with $P(\cdot \mymid s_h,a_h)$. 
When the actions are selected based on policy $\pi$, we can define the $Q$-function and the value function at step $h$ respectively as follows:  
\begin{align}
	Q_h^{\pi}(s,a) \coloneqq \mathbb{E}_{\pi}\left[ \sum_{h'=h}^H r_{h'}  \,\Big|\, (s_h,a_h)=(s,a)\right] \quad \text{and} \quad
	V^{\pi}_h(s) \coloneqq \mathbb{E}_{\pi}\left[ \sum_{h'=h}^H r_{h'} \,\Big|\, s_h = s\right]\nonumber
\end{align}
for any $(s,a)\in \mathcal{S} \times \mathcal{A}$, 
where $\mathbb{E}_{\pi}[\cdot]$ denotes the expectation following $\pi$, i.e., 
we execute $a_{h'}=\pi_{h'}(s_{h'})$ for all $h<h'\leq H$ (resp.~$h\leq h'\leq H$) in the definition of $Q_h^{\pi}$ (resp.~$V_h^{\pi}$). 
The optimal $Q$-function and value function at step $h$ can then be defined respectively as
\begin{align}
	Q_h^*(s,a)  =\max_{\pi}Q^{\pi}_h(s,a) \quad \text{and} \quad V_{h}^*(s) = \max_{\pi}V^{\pi}_h(s), 
	\quad  \forall (s,a)\in \mathcal{S} \times \mathcal{A}.
	\nonumber
\end{align}
%
%for all $$. 
These functions satisfy the Bellman optimality condition in the sense that $V_h^*(s)=\max_a Q_h^*(s,a)$, $\forall s\in \mathcal{S}$, and $Q_{h}^*(s,a) = r(s,a) +\mathbb{E}_{s'\sim P(\cdot |s,a)}[V_{h+1}^*(s')]$, $\forall (s,a)\in \mathcal{S}\times \mathcal{A}$.

\paragraph{The learning process.}
The learning process entails collection of $K$ sample episodes. 
At each episode $k=1,2,\ldots,K$, a policy $\pi^k$ is selected carefully based on the samples collected in the previous $k-1$ episodes;  
the learner then starts from a given initial state $s_1^k$ and executes $\pi^k$ to collect the $k$-th episode 
$\{(s_{h}^{k},a_{h}^{k}, r_h^{k})\}_{1\leq h\leq H}$, where $s_h^{k},a_{h}^{k}$ and $r_h^{k}$  denote respectively the state, action and immediate reward at step $h$ of this episode. 
The learning performance is measured by the total regret 
\begin{align}
	\mathrm{Regret}(K) \coloneqq \sum_{k=1}^K  \Big( V^*_{1} \big(s_1^k\big) - V^{\pi^k}_1 \big(s_1^k \big) \Big), 
	\label{eq:defn-regret-K}
\end{align}
and our ultimate goal is to design a learning algorithm that minimizes the above regret \eqref{eq:defn-regret-K}.

\paragraph{Key assumptions.} 
We now introduce two key assumptions imposed throughout this paper, 
which play a crucial role in determining the minimal regret. 
The first assumption is imposed upon the rewards, requiring the aggregate reward in any episode to be bounded above by $1$ almost surely.
\begin{assumption}[Bounded total rewards] \label{assumr}
	In any episode, we assume that $\sum_{h=1}^H r_h \le 1$ holds almost surely regardless of the policy in use.
\end{assumption}
Compared to the common assumption where the immediate reward at each step is bounded by $1/H$, Assumption~\ref{assumr} is much weaker in that it allows the rewards to be spiky (e.g., we allow the immediate reward at one step to be on the order of 1 with the remaining ones being small). 
The interested reader is referred to \cite{jiang2018open} for more discussions about the above reward assumption.

The second assumption postulates that the transition kernel and the reward function reside within some known low-dimensional subspace, 
a scenario commonly referred to as a linear MDP. 
\begin{assumption}[Linear MDP \citep{jin2019provably}]\label{assuml} 
    Let $\mathcal{B}$ represent the unit $\ell_2$ ball in $\mathbb{R}^d$, 
and let $\{\phi(s,a)\}_{(s,a)\in \mathcal{S}\times\mathcal{A}}\subset \mathcal{B}$ be a set of known feature vectors {such that $\max_{s,a}\|\phi_{s,a}\|_{2}\leq 1$}. 
    Assume that there exist a reward parameter $\theta_{r}\in \mathbb{R}^d$ and  a transition kernel parameter $\mu\in \mathbb{R}^{S\times d}$ such that
\begin{subequations}
\label{eq:assump-linear-MDP}
\begin{align}
	r(s,a)  &= \big\langle \phi(s,a),  \theta_{r} \big\rangle 
	&&\forall (s,a)\in \mathcal{S}\times \mathcal{A} 
 \\ 
	 P(\cdot\mymid s,a) &= \mu\phi(s,a), &&\forall (s,a)\in \mathcal{S}\times \mathcal{A}
 \\ 
	 \|\theta_{r}\|_{2} &\leq \sqrt{d},
\\ 
	\|\mu^{\top}v\|_2 &\leq \sqrt{d},  &&\forall v\in \mathbb{R}^S \text{ obeying } \|v\|_{\infty}\leq 1.
\end{align}
\end{subequations}
\end{assumption}
In words,  Assumption~\ref{assuml} requires both the reward function and the transition kernel 
to be linear combinations of a set of $d$-dimensional feature vectors, 
which enables effective dimension reduction as long as $d\ll SA$.  
 
In comparison, another line of works \cite{jia2020model, ayoub2020model,modi2020sample} focus on the setting of linear mixture MDP below.
\begin{assumption}[Linear Mixture MDP]\label{assum2} 
Let $\{(r_i,P_i)\}_{i=1}^d$ be a group of known reward-transition pairs. Assume that there exists a kernel parameter $\theta\in \mathbb{R}^d$ such that 
\begin{subequations}
\label{eq:assump-linear-Mixture-MDP}
\begin{align}
	r(s,a)  &= \sum_{i=1}^d \theta_i r_i(s,a) 
	&&\forall (s,a)\in \mathcal{S}\times \mathcal{A} 
 \\ 
	 P(\cdot\mymid s,a) &= \sum_{i=1}^d \theta_i P_i(\cdot \mymid s,a), &&\forall (s,a)\in \mathcal{S}\times \mathcal{A}
 \\ 
	 \|\theta\|_{1} &\leq 1.
\end{align}
\end{subequations}
\end{assumption}
Roughly speaking, Assumption~\ref{assum2} requires that the underlying reward-transition pair is a linear combination of $d$ known reward-transition pairs.
Recent work \cite{zhou2022computationally} achieved a  near-tight horizon-regret bound in this setting with a computational efficient algorithm. However, we emphasize that learning a linear MDP is fundamentally harder than learning a linear mixture MDP. The reason is that the only unknown parameter in a linear mixture MDP problem is the hidden kernel $\theta$, which has at most $d$ dimensions. So it is possible to learn $\theta$ to fully express the transition model. 
While in linear MDP, the dimension of unknown parameter $\mu$ scales linearly in the number of states, where it is impossible to recover the transition model. To address this problem, previous works on linear MDP try to learn the transition kernel in some certain direction, e.g.,  $\mu^{\top} v$ for some certain $v\in \mathbb{R}^S$. This approach faces a fundamental problem in sharing samples among difference layers. We refer to Section~\ref{sec:tec} for more discussion.

%As argued in \cite{jin2019provably}, t, because any tabular MDP could be viewed as a linear MDP with dimension $d\leq S\times \mathcal{A}$. 

 \paragraph{Notation.} 
Let us introduce several notation to be used throughout. 
First, we use $\iota$ to abbreviate $\log(2/\delta)$. 
For any $x\in \mathbb{R}^d$ and $\Lambda\in \mathbb{R}^{d\times d}$, 
we define the weighted norm $\|x\|_{\Lambda} \coloneqq \sqrt{x^{\top}\Lambda x}$. 
 Let $[N]$ denote the set $\{1,2,\ldots,N\}$ for a positive integer $N$. Define $\mathcal{B}(x) \coloneqq \{\theta \in \mathbb{R}^d \mid \| \theta\|_2 \leq x\}$ and let $\mathcal{B} \coloneqq \mathcal{B}(1)$  be the unit ball. For two vectors $u,v$ with the same dimension, we say $u\geq v$ (resp.~$u\leq v$) iff $u$ is elementwise no smaller (resp.~larger) than $v$.
For a random variable $X$, we use $\mathrm{Var}(X)$ to denote its variance. 
 For any probability vector $p\in \Delta(\mathcal{S}) $ and any $v=[v_i]_{1\leq i\leq S}\in \mathbb{R}^{S}$, 
 we denote by $\mathbb{V}(p,v)\coloneqq p^{\top} (v^2) - (p^{\top}v)^2$ the associated variance, 
 where $v^2\coloneqq [v_i^2]_{1\leq i\leq S}$ denotes the entrywise square of $v$.  Let $\phi_h^k$ abbreviate $\phi(s_h^k,a_h^k)$ for any proper $(h,k)$.  
Also, we say $(h',k')\leq (h,k)$ iff $h'+k'H\leq h+kH$.
Let $\mathcal{F}_h^k$ denote the $\sigma$-algebra generated by $\{s_{h'}^{k'},a_{h'}^{k'}\}_{(h',k')\leq (h,k)}$.  We employ $\mathbb{E}[\cdot \mymid \tilde{\mu},\theta]$ to denote the expectation 
when the underlying linear MDP is generated by the transition kernel parameter $\tilde{\mu}$ and the reward parameter $\theta$ (cf.~\eqref{eq:assump-linear-MDP}). 
Moreover, let $\Phi$ denote the set of all possible features. Without loss of generality, we assume $\Phi$ is a convex set.

\begin{algorithm}[!t]
\caption{Main Algorithm}
\begin{algorithmic}[1]\label{alg:main}
\STATE{\textbf{Input:} Number of episodes $K$, horizon $H$, feature dimension $d$, confidence parameter $\delta$ }
\STATE{\textbf{Initialization:} $\lambda \leftarrow 1/H^2$, $\epsilon\leftarrow 1/(KH)^4$, $\alpha \leftarrow 150d\sqrt{\log^2((KH)/\delta)}$ }
%\STATE{$\mathbf{b}_h \leftarrow 0$, $\mathbf{\theta}_h\leftarrow 0$ for all $h\in [H]$}
\FOR{$k = 1,2,\ldots, K$}
%\STATE{$V_{H+1}^k(s)\leftarrow 0$, $\forall s$;}

\STATE{$\mathcal{D}^k \leftarrow \{ s_{h'}^{k'},a_{h'}^{k'},s_{h'+1}^{k'} \}_{h'\in [H], k'\in [k-1]}$;}
%\STATE{$(\theta_h^k , \tilde{\theta}_h^k,\Lambda_h^k) \leftarrow \mathtt{HF-Estimator}(\mathcal{D}^k,V_{h+1}^k)$;\label{line:estimator}}

   \STATE{ \it \textcolor{red}{//}   Construct the confidence region for the transition kernel.}
\FOR{$v\in \mathcal{W}_{\epsilon}$}
\STATE{$ (\hat{\theta}^k(v) ,\tilde{\theta}^k(v), \Lambda^k(v)  )\leftarrow \mathtt{HF-Estimator}(\mathcal{D}^k,v)$\label{line:estimator};}
\STATE{$b^k(v,\phi)\leftarrow \alpha \sqrt{\phi^{\top} (\Lambda^k(v))^{-1}\phi} + 4\epsilon;$}
\ENDFOR
\STATE{$\mathcal{U}^k\leftarrow \left\{ \tilde{\mu}\in \mathcal{U} | | \phi^{\top}\tilde{\mu}^{\top}v  -\phi^{\top}\hat{\theta}(v) |\leq b^k(v,\phi), \forall \phi \in \Phi, v\in \mathcal{W}_{\epsilon}\right\}$  }
%\FOR{$k'=1,2,\ldots,k-1$}
%\STATE{$\tilde{\mu}^k\leftarrow \arg\max_{ \tilde{\mu}\in \mathcal{U}^k}V_1^* (s_1^k,\tilde{\mu}) $;}
%\STATE{$\pi^k\leftarrow $ optimal policy with respect to $\tilde{\mu}^k$;}
%\STATE{Execute $\pi^k$ in the $k$-th episode;}
%\ENDFOR
  \STATE{\it \textcolor{red}{//}  Construct the confidence region for the reward function.}
\STATE{$\Theta^k \leftarrow \mathtt{Reward-Confidence}\left(\{\phi_h^{k'}/\sqrt{d}\}_{k'\in [k-1],h\in [H]} , \{r_h^{k'}/\sqrt{d}\}_{k'\in [k-1],h\in [H]} \right)$}
   \STATE{\it \textcolor{red}{//}  Optimistic planning.}
\STATE{$(\mu^k,\theta^k)\leftarrow \arg\max_{\tilde{\mu}\in \mathcal{U}^k, \theta\in \Theta^k }\max_{\pi} \mathbb{E}_{\pi}[\sum_{h=1}^H r_h |\tilde{\mu},\theta]$;}
\STATE{$\pi^k$ be the optimal policy w.r.t. the reward parameter as $\theta^k$ and transition parameter as $\mu^k$;}
\STATE{Play $\pi^k$ in the $k$-th episode;}
 % \STATE{\{{\it Backward planning}\}}
%\STATE{$Q_{H+1}^k(s,a), V_{H+1}^k(s)\leftarrow 0,\forall s,a$;}

%\FOR{$h =1 ,2,\ldots, H$}
%\STATE{Observe $s_h^k$;}
%\STATE{$Q_h(s_h^k,a)\leftarrow \min\{ \phi(s_h^k,a)^{\top}\theta_h^k + \alpha \sqrt{\phi(s_h^k,a)^{\top}(\Lambda_h^k)^{-1}\phi(s_h^k,a)}+4(KH)^3\epsilon ,1 \}$, $\forall a\in \mathcal{A}$;\label{line:update}}
%\STATE{Take action $a_h^k = \arg\max_{a}Q_h^k(s_h^k,a)$;}
%\STATE{Receive reward $r(s_h^k,a)$, and observe $s_{h+1}^k$;}
\ENDFOR
\end{algorithmic}
\end{algorithm}

\begin{algorithm}[t]
\caption{$\mathtt{HF-Estimator}$}
\begin{algorithmic}[1]\label{alg:vlambdae}
\STATE{\textbf{Input :} A group of samples $\mathcal{D}:=\{ s_i,a_i,s_i'\}_{i=1}^n$, value function $v\in \mathbb{R}^S$;   }
\STATE{\textbf{Initialization:} $\lambda \leftarrow 1/H^2$, $\alpha \leftarrow 150d\sqrt{\log^2((KH)/\delta)}$, $\phi_i \leftarrow \phi(s_i,a_i), 1\leq i \leq n$, $\epsilon\leftarrow 1/(KH)^4$;}
\STATE{$\sigma^2_1 \leftarrow 4$;}
\FOR{$i = 2,3,\ldots,n+1$}
\STATE{$\Lambda_{i-1} \leftarrow \lambda \mathbf{I}+ \sum_{i'=1}^{i-1}\phi_{i'}^{\top}\phi_{i'}/\sigma^2_{i'}$;}
\STATE{$\tilde{b}_{i-1} \leftarrow \sum_{i'=1}^{i-1}\frac{v^2(s'_{i'})}{\sigma^2_{i'}}\phi_{i'}$, $\tilde{\theta}_{i-1}\leftarrow \Lambda_{i-1}^{-1}\tilde{b}_{i-1}$;}
\STATE{$b_{i-1} \leftarrow  \sum_{i'=1}^{i-1}\frac{v(s'_{i'})}{\sigma^2_{i'}}\phi_{i'}$, $\theta_{i-1}\leftarrow \Lambda_{i-1}^{-1}b_{i-1}$;}
\STATE{$\sigma^2_{i}\leftarrow  \phi_i^{\top} \tilde{\theta}_{i-1}- (\phi_i^{\top}\theta_{i-1})^2 + 16\alpha \sqrt{\phi_i^{\top} (\Lambda_{i-1})^{-1}\phi_i} +4\epsilon, 
$;\label{line:sigma}}
%\STATE{$\sigma^2_i\leftarrow \max\{ \sigma_i^2, \frac{1}{H^2} \}$;}
\ENDFOR
\STATE{$\theta\leftarrow \Lambda_{n}^{-1}b_{n}$, $\tilde{\theta}\leftarrow \Lambda_n^{-1}\tilde{b}_n$, $\Lambda\leftarrow \Lambda_{n}$;}
\STATE{\textbf{Return:} $(\theta,\tilde{\theta},\Lambda)$;}
\end{algorithmic}
\end{algorithm}

\section{Technique Overview}\label{sec:tec}

In this section, we first discuss the hardness of horizon-free bounds for linear MDP, and then
introduce the high-level ideas of our approach. 
To simplify presentation, we focus on the regret incurred by learning the unknown transition dynamics.

\subsection{Technical Challenge}
\paragraph{Least-square regression in the linear MDP  problem.} \cite{jin2020provably} proposed the first efficient algorithm (LSVI-UCB) for the linear MDP problem. In this method, for each $h\in [H]$, the learner maintains an estimation of $V_{h+1}$, and constructs optimistic estimators of $Q_h(s,a):= r(s,a)+P_{s,a}^{\top}V_{h+1}$. Since the reward $r$ is assumed to be known, it suffices to estimate $P_{s,a}^{\top}V_{h+1}= (\phi(s,a))^{\top}\mu^{\top}V_{h+1}$. By defining $\theta_{h+1}:= \mu^{\top}V_{h+1}$, we can estimate $(\phi(s,a))^{\top} \theta_{h+1}$ with least-square regression because all state-action pairs share the same kernel $\theta_{h+1}$. This task is basically the same as a linear bandit problem, except for that additional factors are needed due to uniform bound over all possible choices $V_{h+1}$.

To obtain horizon-free regret bound, a common approach is to design estimators for $P_{s,a}^{\top}V_{h+1}$ with smaller confidence intervals. In this way, we can choose a smaller bonus to keep the optimism, and the regret is also reduced since the leading term in the regret is the sum of bonuses. 

Recent work \cite{zhou2022computationally} made progress in this direction by designing a variance-aware estimators for the linear regression problem. Roughly speaking, given a groups of samples $\{\phi_i,v_i\}_{i=1}^n$ such that (\romannumeral1) $v_{i}=\phi_i^{\top}\theta+\epsilon_i$, $\forall i\in [n]$; (\romannumeral2) $\mathbb{E}[\epsilon_i| \{\phi_j\}_{j=1}^i, \{\epsilon_j\}_{j=1}^{i-1}]=0$ and $\mathbb{E}[\epsilon^2_i| \{\phi_j\}_{j=1}^i, \{\epsilon_j\}_{j=1}^{i-1}]=\sigma_i^2$, $\forall i\in [n]$, with the method in \cite{zhou2022computationally}, 
the width of the confidence interval of $\phi^{\top}\theta$ is roughly 
\begin{align}\tilde{O}\left( \mathrm{poly}(d)\sqrt{\phi^{\top}\Lambda^{-1}\phi} \right),\label{eq:cxx1}\end{align}
where $\Lambda= \lambda \mathbf{I}+ \sum_{i=1}^n \frac{\phi_i\phi_i^{\top}}{\sigma_i^2}$ and $\lambda$ is some proper regularization parameter  (See Lemma~\ref{lemma:take1} in Appendix~\ref{app:teclemma}).

\paragraph{Main technical challenge: Variance-aware estimators coupled with inhomogeneous value functions.} While the transition kernel is assumed to be time-homogeneous, the value function and the policy can be \emph{time-inhomogeneous} across different steps. 
 Although the confidence width in \eqref{eq:cxx1} seems nice, it poses additional difficulty to bound the sum of bonuses due to \emph{time-inhomogeneous} value functions.

%Existing works on tabular MDP and linear mixture MDPs attempted to alleaviate this problem by explcitly learning the transition kernel and calculating the value function based on the learned transition model.
%However, this approach cannot be readily used for linear MDPs as the model space can be exceedingly large.

Below we give more technical details to elucidate this technical challenge.
To simplify the problem, we assume that the learner is informed of both the reward function and the optimal value function $\{V_h^*\}_{h\in [H]}$. Note that the arguments below can be extended to accommodate 
unknown $\{V_{h}^*\}_{h\in [H]}$ as well by means of proper exploitation of the linear MDP structure and a discretization method 
(i.e., applying a union bound over all possible optimal value functions over a suitably discretized set). 
%\simon{explain why we want assume this?} \textcolor{blue}{We assume this to simplify the problem. The major difficulty still exists even if we make this assumption.} 

Let $\theta_h^*  = \mu^{\top}V_{h+1}^*$. Then it reduces to learning $H$ contextual bandit problems with hidden parameter as $\{\theta_h^*\}_{h=1}^{H}$. 
To remove the polynomial factors of $H$, it is natural to  share samples over different layers. That is, we need to use all the samples along the trajectory $\{s_{h'},a_{h'},s_{h'+1}\}_{h'=1}^{H}$ to estimate the value of $\phi^{\top}\theta_{h}^*$.

  To solve the $h$-th linear bandit problem, following \eqref{eq:cxx1}, we could get a regret bound of $ \mathrm{Regret}_h(K):=\tilde{O}\left(\sum_{k=1}^K \sqrt{  (\phi_h^k)^{\top}  (\Lambda^k (V_{h+1}^*) )^{-1}\phi_h^k  }\right)$. 
 Here $\Lambda^k(v)= \lambda \mathbf{I} +\sum_{k'=1}^{k-1} \sum_{h'=1}^{H}\frac{\phi_{h'}^{k'}(\phi_{h'}^{k'})^{\top}}{  \left(\sigma_{h'}^{k'} ( v)\right)^2}$ with $\left(\sigma_{h'}^{k'} (v)\right)^2$ as an upper bound for the variance $\mathbb{V}(P_{s_{h'}^{k'},a_{h'}^{k'}},v) $ for $v\in \mathbb{R}^{S}$. Taking sum over $h$, the resulting regret bound is roughly 
\begin{align}
\sum_{k=1}^K \min\left\{  \sum_{h=1}^H \sqrt{  (\phi_h^k)^{\top}  (\Lambda^k(V_{h+1}^*))^{-1}\phi_h^k  }  ,1\right\}.\label{eq:cb0}
\end{align}
We remark that if $V_{h}^* $ is homogeneous in $h$, i.e., there exists $V^*$ such that $V_h^* = V^*$ for any $h\in [H]$,  we could use Cauchy's inequality %and Elliptical potential lemma 
to bound \eqref{eq:cb0} by\footnote{Here we omit a lower order term.}
\begin{align}
\sqrt{\sum_{k=1}^K \min \left\{   \sum_{h=1}^H (\phi_h^k/(\sigma_h^k(V^*)) )^\top (\Lambda^k(V^*))^{-1} (\phi_h^k/\sigma_h^k(V^*))    ,1\right\}  } \cdot \sqrt{\sum_{k=1}^K \sum_{h=1}^H (\sigma_h^k(V^*))^2  }.\label{eq:cb}
\end{align}

Noting that 
\begin{align}
\Lambda^{k+1}(V^*) = \Lambda^k(V^*) + \sum_{h'=1}^H \frac{ \phi_{h'}^{k}(\phi_{h'}^{k})^{\top} }{ \left(\sigma_{h'}^{k'} ( V^*)\right)^2 }, \label{eq:consistent}
\end{align}
we can further use the elliptical potential lemma (Lemma~\ref{lemma:epl}) to bound the first term in \eqref{eq:cb}, and the total variance lemma for  MDPs to bound the second term in \eqref{eq:cb}. As a result, we can easily bound \eqref{eq:cb0} by $\tilde{O}(\mathrm{poly}(d)\sqrt{K})$. 

However, the arguments above cannot work when $V_{h}^*$ depends on $h$. In such cases, the first term in \eqref{eq:cb0} would be
\begin{align}
\sqrt{\sum_{k=1}^K \min \left\{   \sum_{h=1}^H  (\phi_h^k /\sigma_h^k(V_h^*))^{\top} (\Lambda^k(V^*_{h+1}))^{-1} (\phi_h^k/\sigma_h^k(V_h^*)), 1   \right\}}.
\end{align}

To invoke elliptical potential lemma, we need $\Lambda^{k+1}(V_{h+1}^*)-\Lambda^k(V_{h+1}^*) = \sum_{h'=1}^H \frac{\phi_{h'}^k (\phi_{h'}^k)^{\top}}{(\sigma_{h'}^k(V_{h'+1}^*))^2}$,
which does not hold since $\sigma_{h'}^k(V_{h'+1}^*)\neq \sigma_{h'}^k(V_{h+1}^*)$.

In comparison, for tabular MDP, the variance aware bonus has a simple form of $\sqrt{\frac{\mathbb{V}(P_{s,a},V^*_{h+1})}{N}}$, so that one can invoke Cauchy's inequality to bound the sum of bonuses; for linear mixture MDP, because there is only one kernel parameter $\theta$ and one information matrix, it suffices to analyze like \eqref{eq:cb} and \eqref{eq:consistent}.

\subsection{Our Methods} 
In high-level idea, by noting that the main obstacle is the \emph{time-inhomogeneous} value function, we aim to prove that the value function $\{V^*_h\}_{h=1}^H$ could be divided into several groups such that in each group, the value functions are similar to each other measured by the variance.

\paragraph{Technique 1: a uniform upper bound for the variances.} 
We consider using a uniform upper bound $(\bar{\sigma}_{h'}^{k'})^2:=\max_{h\in [H]} (\sigma_{h'}^{k'}(V_{h+1}^*))^2$ to replace $(\sigma_{h'}^{k'}(V^*_{h+1}))^2$ when computing $\Lambda^k(V_{h+1}^*)$. 
That is, by setting $\bar{\Lambda}^k = \lambda \mathbf{I} + \sum_{k'=1}^{k-1}\sum_{h'=1}^H \frac{ \phi_{h'}^{k'} (\phi_{h'}^{k'})^{\top} }{(\bar{\sigma}_{h'}^{k'})^2}\preccurlyeq  \Lambda^k(V_{h+1}^*)$ for any $h\in [H]$, we can bound \eqref{eq:cb0} as below:
\begin{align}
  & \sum_{k=1}^K \min\left\{  \sum_{h=1}^H \sqrt{  (\phi_h^k)^{\top}  (\Lambda^k(V_{h+1}^*))^{-1}\phi_h^k  }  ,1\right\} \nonumber
  \\ & \leq \sum_{k=1}^K \min\left\{  \sum_{h=1}^H \sqrt{  (\phi_h^k)^{\top}  (\bar{\Lambda}^k))^{-1}\phi_h^k  }  ,1\right\} \nonumber
  \\ & \approx \sqrt{\sum_{k=1}^K \min \left\{ \sum_{h=1}^H ( \phi_h^k/\bar{\sigma}_h^k)^{\top} (\bar{\Lambda}^k)^{-1}(\phi_h^k/\bar{\sigma}_h^k)  ,1\right\} }\cdot \sqrt{\sum_{k=1}^K \sum_{h=1}^H  {(\bar{\sigma}_h^k)^2}}.
\end{align}

With the elliptical potential lemma (Lemma~\ref{lemma:epl}), we have that 
$$\sum_{k=1}^K \min \left\{ \sum_{h=1}^H ( \phi_h^k/\bar{\sigma}_h^k)^{\top} (\bar{\Lambda}^k)^{-1}(\phi_h^k/\bar{\sigma}_h^k)  ,1\right\}  = \tilde{O}(\sqrt{d}).$$ So it suffices to deal with $\sum_{k=1}^K \sum_{h=1}^H  {(\bar{\sigma}_h^k)^2}$. For simplicity, we assume that $ {(\sigma_{h}^k(v))^2}$  is exactly $\mathbb{V}(P_{s_h^k,a_h^k},v)$  and consider to bound $\sum_{k=1}^K \sum_{h=1}^H  \max_{h'\in [H]}\mathbb{V}(P_{s_h^k,a_h^k}, V_{h'+1}^*)$. 

Noting that $\mathbb{V}(P_{s,a},v)$ could be written as $\phi(s,a)^{\top} (\theta(v^2)) - (\phi(s,a)^{\top}\theta(v))^2$, which is a linear function of the matrix $\left[\begin{array}{cc}
  \phi(s,a)\phi^{\top}(s,a)   &  \phi(s,a)\\
   \phi^{\top}(s,a) & 1
\end{array}\right] $, we can bound $\sum_{k=1}^K\sum_{h=1}^H \max_{h'\in [H]}\mathbb{V}(P_{s_h^k,a_h^k},V_{h'+1}^*)$ by $2(d+1)^2\max_{h'\in [H]}\sum_{k=1}^K \sum_{h=1}^H \mathbb{V}(P_{s_h^k,a_h^k},V_{h'+1}^*)$ with a useful technical lemma (See Lemma~\ref{lemma:tool5}.)
%\simon{we can move this lemma to appendix but mention we have a useful technical lemma.}

As a result, it suffices to bound $\sum_{k=1}^K \sum_{h=1}^H \mathbb{V}(P_{s_h^k,a_h^k},V_{h'+1}^*)$ for each $h'\in [H]$.  However, because $V_{h'+1}^*$ can vary significantly when $h'$ is closed to $H$, $\sum_{k=1}^K \sum_{h=1}^H \mathbb{V}(P_{s_h^k,a_h^k},V_{h'+1}^*)$ might be large in the worst case. We consider the toy example below.

\begin{example} Fix some $\epsilon>0$.
Let $\mathcal{S}:=\{s_1,s_2,s_3 , z\}$, $\mathcal{A} = \{a_1,a_2\}$. Let $P_{s_1,a_1}=P_{s_2,a_1}  = [\frac{1}{2}-\epsilon, \frac{1}{2}-\epsilon,\epsilon,0 ]^{\top}$, $r(s_1,a_1)=r(s_2,a_1)=0$, $P_{s_1,a_2} =P_{s_2,a_2} = [0,0,0,1]^{\top}$, $r(s_1,a_2)=\frac{1}{2}$, $r(s_2,a_2)=0$, $P_{s_3,a_1}=P_{s_3,a_2} = [0,0,0,1]^{\top}$, $r(s_3,a_1)=r(s_3,a_2)=1$, $P_{z,a_1}=P_{z,a_2}=[0,0,0,1]^{\top}$, and $r(z,a_1)=r(z,a_2)=0$. 
\end{example}
In this toy example, we have two frequent states $\{s_1,s_2\}$, one transient state $\{s_3\}$ with reward $1$ and one ending state $z$ with no reward. 
The transition dynamics at $\{s_1,s_2\}$ is the same, but one can get reward $\frac{1}{2}$ in one step by taking action $a_2$ at $s_1$.
Suppose $H>> \frac{1}{\epsilon}$ and $h\leq \frac{H}{2}$, then the optimal action for $\{s_1,s_2\}$  at the $h$-th step should be $a_1$, and $V_{h}^*(s_1)\approx V_{h}^*(s_2)\approx 1$. On the other hand, it is easy to observe $V_{H}^*(s_1)=\frac{1}{2}$ and $V_{H}^{*}(s_2)=0$. Let the initial state be $s_1$. Following the optimal policy, we have $\mathbb{E}\left[\sum_{h=1}^{\frac{H}{2}}\mathbb{V}(P_{s_h^k,a_h^k},V_{H}^*) \right]= \Omega(\frac{1}{\epsilon})>>1$ when choosing $\epsilon$ small enough.

%To address this problem, we consider to divide $[H]$ into disjoint segments as below.

\paragraph{Technique 2: bounding the total variation.}
%\simon{we should highlight more on $\|V_h^*-V_{h+1}^*\|_{\infty}$ bound. Maybe call this technique 2, and doubling segments the technique 3.} \textcolor{blue}{Wit}
Direct computation shows that for $1\leq h_1< h_2 \leq [H]$,
\begin{align}
\sum_{k=1}^K \sum_{h=h_1}^{h_2} \mathbb{V}(P_{s_h^k,a_h^k},V_{h'+1}^*) = \tilde{O}(K + K(h_2-h_1+1) \|V^*_{h'}-V_{h'+1}^*\|_{\infty}).\label{eq:vvbound}
\end{align} 
%which might depends on $H$ when $\|V_{h'}^* - V_{h'+1}^*\|_{\infty}=\Theta(1)$. 

Let $ {l_h} = \|V_{h}^*-V_{h+1}^*\|_{\infty}$. It is easy to observe that $l_h \leq l_{h+1}$ for $1\leq h \leq H-1$ since the Bellman operator  $\Gamma$ is a contraction , i.e.,  $\|\Gamma (v_1-v_2)\|_{\infty}\leq \|v_1-v_2\|_{\infty}$ for any $v_1,v_2\in \mathbb{R}^S$. So we can obtain $l_h \leq \frac{\sum_{h'=1}^{H-1} l_{h'}}{H-h+1}$. For tabular MDP, it is easy to bound $\sum_{h=1}^H l_h \leq S$ since $\|V_h^*-V_{h+1}^*\|_{\infty}\leq \sum_{s}(V_h(s)-V_{h+1}(s))$.  As a generalization to linear MDP, by Lemma~\ref{lemma:tool5} we have that
\begin{align}
\sum_{h=1}^{H-1} l^*_h \leq \sum_{h=1}^{H-1} \max_{\phi \in \Phi} \phi^{\top}\mu^{\top} (V_{h+1}-V_{h+2}) \leq \max_{\phi\in\Phi} 2d\phi^{\top}\sum_{h=1}^{H-1}\mu^{\top} (V_{h+1}-V_{h+2}) \leq 2d.
\end{align}
As a result, $l^*_h \leq \frac{2d}{H-h+1}$.

\paragraph{Technique 3: doubling segments.}
By choosing $h_1 = \frac{H}{2}+1$ and $h_2 = H$ in \eqref{eq:vvbound}, for $h'\in [h_1,h_2]$,
\begin{align}
\sum_{k=1}^K \sum_{h=h_1}^{h_2} \mathbb{V}(P_{s_h^k,a_h^k},V_{h'+1}^*) = \tilde{O}(K + K(h_2-h_1+1) \|V^*_{h'}-V_{h'+1}^*\|_{\infty}) =\tilde{O}(Kd). \nonumber
\end{align}
This inspires us to divide $[H]$ several segments $ [H]= \cup_{i}\mathcal{H}_i$ with $\mathcal{H}_{i} = \{h| H - \frac{H}{2^{i-1}}+1\leq h\leq   H - \frac{H}{2^i}\}$ and $\mathcal{H}_{\log_2(H)+1}=\{H\}$\footnote{We assume $\log_2(H)$ is an integer without loss of generality.}. Consequently, for any $i$ and $h' \in \mathcal{H}_i$, using \eqref{eq:vvbound} and the fact that $l_{h'}^* \leq \frac{2d}{H-h'+1} \leq \frac{2^{i+1}d}{H}$,
 $\sum_{k=1}^K \sum_{h\in \mathcal{H}_i}\mathbb{V}(P_{s_h^k,a_h^k},V_{h'+1}^*)  = \tilde{O}(Kd).$

Note that we only bound $\sum_{k=1}^K \sum_{i=1}^{\log_2(H)+1}\max_{h'\in \mathcal{H}_i}\sum_{h\in \mathcal{H}_i}\mathbb{V}(P_{s_h^k,a_h^k},V_{h'+1}^*)$, which does not imply any bound for $\max_{h'\in [H]}\sum_{k=1}^K \sum_{h=1}^H \mathbb{V}(P_{s_h^k,a_h^k},V_{h'+1}^*)$.   {Recall that our initial target is  to bound $\sum_{k=1}^K \min\left\{  \sum_{h=1}^H \sqrt{  (\phi_h^k)^{\top}  (\Lambda^k(V_{h+1}^*))^{-1}\phi_h^k  }  ,1\right\}$. A natural idea is to group $h\in \mathcal{H}_i$ for each $i$ to avoid the term  $\max_{h'\in [H]}\sum_{k=1}^K \sum_{h=1}^H \mathbb{V}(P_{s_h^k,a_h^k},V_{h'+1}^*)$.} In other words, we  turn to bound $\sum_{k=1}^K \min\left\{ \sum_{h\in \mathcal{H}_i}\sqrt{  (\phi_h^k)^{\top}(\Lambda^k(V_{h+1}^*))^{-1}\phi_h^k } ,1\right\}$ for each $i$ separately. More precisely, for fixed $i$, we let $(\bar{\sigma}_{h'}^{k'})^2 = \max_{h\in \mathcal{H}_i}(\sigma_{h'}^{k'}(V_{h+1}^*))^2 $, and $\bar{\Lambda}^k = \lambda \mathbf{I}+\sum_{k'=1}^{k-1}\sum_{h'\in \mathcal{H}_i}\frac{ \phi_{h'}^{k'}(\phi_{h'}^{k'})^{\top} }{(\bar{\sigma}_{h'}^{k'})^2}$. With the arguments above, we have that
\begin{align}
& \sum_{k=1}^K \min\left\{ \sum_{h\in \mathcal{H}_i}\sqrt{  (\phi_h^k)^{\top}(\Lambda^k(V_{h+1}^*))^{-1}\phi_h^k },1 \right\}  \nonumber
\\ & \leq \sqrt{\sum_{k=1}^K \min\left\{  \sum_{h\in \mathcal{H}_i} (\phi_{h}^k/\bar{\sigma}_{h}^k) (\bar{\Lambda}^k)^{-1} (\phi_h^k/\bar{\sigma}_h^k)  ,1 \right\} }\cdot \sqrt{\sum_{k=1}^K \sum_{h\in \mathcal{H}_i} {(\bar{\sigma}_h^k)^2}} + \tilde{O}(d)\nonumber
\\ & =\tilde{O}(\sqrt{Kd^4}).
\end{align}

%Obviously this is possible since  the value function in the last several layers could change significantly.

%For each bandit problem, we can use the whole dataset $\mathcal{D} : = $ to obtain a 

\section{Algorithm}

In this section, we introduce  Algorithm~\ref{alg:main}. The algorithm is based on model-elimination. 
At each episode $k=1,2,\ldots,K$, we maintain $\mathcal{U}^k$ as confidence region of $\mu$ and $\Theta^k$ as confidence region for $\theta_r$. Then we select the optimistic transition model and reward function from $\mathcal{U}^k\times \Theta^k$ and then execute the corresponding optimal policy. The key step is how to construct $\mathcal{U}^k$.
Inspired by recent work \cite{zhou2022computationally}, we consider the weighted least square regression to estimate the value function and corresponding variance, which is presented in Algorithm~\ref{alg:vlambdae}.  
We also borrow VOFUL in \cite{zhang2021variance} to construct the confidence region for $\theta_{r}$.

Recall that  $\mathcal{B}(2\sqrt{d})= \{\theta \in \mathbb{R}^d | \|\theta\|_2 \leq 2\sqrt{d}\}$. For fixed $\epsilon>0$, there exists an $\epsilon$-net $\mathcal{B}_{\epsilon}(2\sqrt{d})$ w.r.t. $L_{\infty}$ for $\mathcal{B}(2\sqrt{d})$ such that $|\mathcal{B}_{\epsilon}(2\sqrt{d})|\leq O((4\sqrt{d}/\epsilon)^d)$.  By Assumption~\ref{assuml}, for any $v\in \mathbb{R}^S$  such that $\|v\|_{\infty}\leq 1$, it holds that $\|\mu^{\top}v\|_{2}\leq \sqrt{d}$. Therefore, for any MDP such that Assumption~\ref{assumr} and \ref{assuml} holds, its optimal value function is in the set
$$\mathcal{W}:=\{v\in \mathbb{R}^S | \exists \theta\in \mathcal{B}(2\sqrt{d}), v(s) = \max\{\min\{\max_{a}\phi^{\top}(s,a)\theta ,1\},0\} ,\forall s\in \mathcal{S} \}.$$

Define 
${\mathcal{W}_{\epsilon}}=\left\{ v\in \mathbb{R}^S| \exists\theta \in \mathcal{B}_{\epsilon}(2\sqrt{d}), v(s) = \max\{\min\{\max_{a}\phi^{\top}(s,a)\theta ,1\},0\} ,\forall s\in \mathcal{S}\right\}.$
For  fixed $\theta \in \mathcal{B}(2\sqrt{d})$ and $s\in \mathcal{S}$, the function $\max\left\{\min\left\{ \max_{a}\phi(s,a)^{\top}\theta   ,1\right\},0\right\}$ is $O(1)$-Lipschtiz continuous w.r.t $L_{\infty}$ norm. As a result, $\mathcal{W}_{\epsilon}$ is an $\epsilon$-net w.r.t. $L_{\infty}$ norm of $\mathcal{W}$. Besides, the size of $\mathcal{W}_{\epsilon}$ is bounded by $|\mathcal{W}_{\epsilon}|=O( (4\sqrt{d}/\epsilon)^{d})$.

\textbf{Confidence region for the transition kernel.} Fix a group of sequential samples $\{\phi_i\}_{i=1}^n$ and a value function $v\in \mathcal{W}_{\epsilon}$. Fix $\phi\in \Phi$ and let $\theta(v) = \mu^{\top}v$. We aim to construct a confidence interval from $\phi^{\top}\mu^{\top}v$, and then eliminate all the transition kernels $\tilde{\mu}$ which fails to satisfy the confidence interval for some $v$ and $\phi$. To obtain variance-aware confidence interval, we need to compute the variance to feed the weight least-square estimator in \cite{zhou2022computationally}. For this purpose, for the $i$-th variance $\mathbb{V}(\mu \phi_i, v)$, we construct $\sigma_i^2$ such that $\sigma_i^2\geq \mathbb{V}(\mu \phi_i, v) $ and the error $\sigma_i^2 - \mathbb{V}(\mu \phi_i, v)$ is well controlled.
 To compute $\sigma_i^2$, we need to  estimate
$\phi_i^{\top}\theta(v^2)$ and $\phi_i^{\top}\theta(v)$ using the first $i-1$-th samples, which requires the knowledge of $\mathbb{V}(\mu\phi_{i'},v^2)$ for $i'\leq i-1 $. To address this problem, \cite{zhou2022computationally} recursively estimated the $2^{m}$-th order momentum for $m=1,2,\ldots,\log_2(H)$. In comparison, by the fact that $\mathbb{V}(\mu\phi_{i'},v^2)\leq 4\mathbb{V}(\mu\phi_{i'},v) $ (see Lemma~\ref{lemma:add1}), we can use $4\sigma_{i'}^2$ as an upper bound for $\mathbb{V}(\mu\phi_{i'},v^2)$.  
%The final output of Algorithm~\ref{alg:vlambdae} is 

%\simon{need to give a short description of HF-estimator. Why this is called HF-estimator?} In Line~\ref{line:estimator} of Algorithm~\ref{alg:main}, we invoke a sub-algorithm to estimate the hidden parameter $\theta = \mu^{\top}v$ and the weighted information matrix $\Lambda$  given a group of samples $\{s_i,a_i,s_i'\}_{i=1}^n$ and a value function $v\in \mathbb{R}^S$ as input. The precise algorithm is presented in Algorithm~\ref{alg:vlambdae}.  \citet{zhou2022computationally} proposed a similar estimator called $\mathtt{HOME}$, which estimated the high-order momentum in a recursive way. \simon{need to discuss what is the difference.}

\textbf{Confidence region for the reward parameter.}
To estimate the reward parameter $\theta_r$, we invoke VOFUL in \cite{zhang2021variance}. We remark that the randomness in reward is independent of the randomness in transition dynamics, so that learning the transition dynamic does not help to estimate the variance of reward.  More precisely, the variance of $R(s,a)$ could be a non-linear function in $\phi(s,a)$, while the variance of $V_{h}^*(s')$ with $s'\sim P(\cdot|s,a)$ must be a linear function in $\phi(s,a)$. In Appendix~\ref{app:voful}, 
we present 
VOFUL and summarize some useful properties  to bound the error due to uncertainty of reward parameter.

% \STATE{$\Lambda^{(m)}\leftarrow \lambda \mathbf{I}+  \sum_{i=1}^n \phi_i\phi_i^{\top}/(\sigma_i^{(m)})^2 $;}

%\FOR{$i = 1,2,\ldots,n$}\STATE{$\mathcal{D}_i \leftarrow \{s_{i'},a_{i'},s'_{i'}\}_{i'=1}^i$;}\ENDFOR

%\section{Conclusion}
%\input{conclusion.tex}

\section*{Acknowledgement}Y.~Chen is supported in part by the Alfred P.~Sloan Research Fellowship, the Google Research Scholar Award, the AFOSR grants FA9550-19-1-0030 and FA9550-22-1-0198, the ONR grant N00014-22-1-2354,  and the NSF grants CCF-2221009 and CCF-1907661.  JDL acknowledges support of the ARO under MURI Award W911NF-11-1-0304,  the Sloan Research Fellowship, NSF CCF 2002272, NSF IIS 2107304,  NSF CIF 2212262, ONR Young Investigator Award, and NSF CAREER Award 2144994. SSD acknowledges the support of NSF IIS 2110170, NSFDMS 2134106, NSF CCF 2212261, NSF IIS 2143493, NSF CCF 2019844, NSF IIS 2229881.

\bibliography{ref}

\begin{thebibliography}{42}
\providecommand{\natexlab}[1]{#1}
\providecommand{\url}[1]{\texttt{#1}}
\expandafter\ifx\csname urlstyle\endcsname\relax
  \providecommand{\doi}[1]{doi: #1}\else
  \providecommand{\doi}{doi: \begingroup \urlstyle{rm}\Url}\fi

\bibitem[Abbasi-Yadkori et~al.(2011)Abbasi-Yadkori, P{\'a}l, and
  Szepesv{\'a}ri]{abbasi2011improved}
Yasin Abbasi-Yadkori, D{\'a}vid P{\'a}l, and Csaba Szepesv{\'a}ri.
\newblock Improved algorithms for linear stochastic bandits.
\newblock \emph{Advances in neural information processing systems}, 24, 2011.

\bibitem[Agarwal et~al.(2022)Agarwal, Jin, and Zhang]{agarwal2022vo}
Alekh Agarwal, Yujia Jin, and Tong Zhang.
\newblock Vo $ q $ l: Towards optimal regret in model-free rl with nonlinear
  function approximation.
\newblock \emph{arXiv preprint arXiv:2212.06069}, 2022.

\bibitem[Agrawal \& Jia(2017)Agrawal and Jia]{agralwal2017optimistic}
Shipra Agrawal and Randy Jia.
\newblock Optimistic posterior sampling for reinforcement learning: worst-case
  regret bounds.
\newblock In \emph{Advances in Neural Information Processing Systems}, pp.\
  1184--1194, 2017.

\bibitem[Auer(2002)]{auer2002using}
Peter Auer.
\newblock Using confidence bounds for exploitation-exploration trade-offs.
\newblock \emph{Journal of Machine Learning Research}, 3\penalty0
  (Nov):\penalty0 397--422, 2002.

\bibitem[Ayoub et~al.(2020)Ayoub, Jia, Szepesvari, Wang, and
  Yang]{ayoub2020model}
Alex Ayoub, Zeyu Jia, Csaba Szepesvari, Mengdi Wang, and Lin~F Yang.
\newblock Model-based reinforcement learning with value-targeted regression.
\newblock In \emph{Proceedings of the 37th International Conference on Machine
  Learning}, 2020.

\bibitem[Azar et~al.(2017)Azar, Osband, and Munos]{azar2017minimax}
Mohammad~Gheshlaghi Azar, Ian Osband, and R{\'e}mi Munos.
\newblock Minimax regret bounds for reinforcement learning.
\newblock In \emph{Proceedings of the 34th International Conference on Machine
  Learning}, pp.\  263--272, 2017.

\bibitem[Brafman \& Tennenholtz(2003)Brafman and Tennenholtz]{brafman2002r}
Ronen~I. Brafman and Moshe Tennenholtz.
\newblock R-max - a general polynomial time algorithm for near-optimal
  reinforcement learning.
\newblock \emph{J. Mach. Learn. Res.}, 3\penalty0 (Oct):\penalty0 213--231,
  March 2003.
\newblock ISSN 1532-4435.

\bibitem[Chen et~al.(2021)Chen, Jafarnia-Jahromi, Jain, and
  Luo]{chen2021implicit}
Liyu Chen, Mehdi Jafarnia-Jahromi, Rahul Jain, and Haipeng Luo.
\newblock Implicit finite-horizon approximation and efficient optimal
  algorithms for stochastic shortest path.
\newblock \emph{Advances in Neural Information Processing Systems}, 34, 2021.

\bibitem[Chu et~al.(2011)Chu, Li, Reyzin, and Schapire]{chu2011contextual}
Wei Chu, Lihong Li, Lev Reyzin, and Robert Schapire.
\newblock Contextual bandits with linear payoff functions.
\newblock In \emph{Proceedings of the Fourteenth International Conference on
  Artificial Intelligence and Statistics}, pp.\  208--214. JMLR Workshop and
  Conference Proceedings, 2011.

\bibitem[Dani et~al.(2008)Dani, Hayes, and Kakade]{dani2008stochastic}
Varsha Dani, Thomas~P Hayes, and Sham~M Kakade.
\newblock Stochastic linear optimization under bandit feedback.
\newblock 2008.

\bibitem[Dann et~al.(2018)Dann, Jiang, Krishnamurthy, Agarwal, Langford, and
  Schapire]{dann2018oracle}
Christoph Dann, Nan Jiang, Akshay Krishnamurthy, Alekh Agarwal, John Langford,
  and Robert~E. Schapire.
\newblock On oracle-efficient {PAC}-{RL} with rich observations.
\newblock In \emph{Advances in Neural Information Processing Systems}, 2018.

\bibitem[Domingues et~al.(2021)Domingues, M{\'e}nard, Kaufmann, and
  Valko]{domingues2021episodic}
Omar~Darwiche Domingues, Pierre M{\'e}nard, Emilie Kaufmann, and Michal Valko.
\newblock Episodic reinforcement learning in finite mdps: Minimax lower bounds
  revisited.
\newblock In \emph{Algorithmic Learning Theory}, pp.\  578--598. PMLR, 2021.

\bibitem[Faury et~al.(2020)Faury, Abeille, Calauz{\`e}nes, and
  Fercoq]{faury2020improved}
Louis Faury, Marc Abeille, Cl{\'e}ment Calauz{\`e}nes, and Olivier Fercoq.
\newblock Improved optimistic algorithms for logistic bandits.
\newblock In \emph{International Conference on Machine Learning}, pp.\
  3052--3060. PMLR, 2020.

\bibitem[He et~al.(2022)He, Zhao, Zhou, and Gu]{he2022nearly}
Jiafan He, Heyang Zhao, Dongruo Zhou, and Quanquan Gu.
\newblock Nearly minimax optimal reinforcement learning for linear markov
  decision processes.
\newblock \emph{arXiv preprint arXiv:2212.06132}, 2022.

\bibitem[Jia et~al.(2020)Jia, Yang, Szepesvari, and Wang]{jia2020model}
Zeyu Jia, Lin Yang, Csaba Szepesvari, and Mengdi Wang.
\newblock Model-based reinforcement learning with value-targeted regression.
\newblock In \emph{Learning for Dynamics and Control}, pp.\  666--686. PMLR,
  2020.

\bibitem[Jiang \& Agarwal(2018)Jiang and Agarwal]{jiang2018open}
Nan Jiang and Alekh Agarwal.
\newblock Open problem: The dependence of sample complexity lower bounds on
  planning horizon.
\newblock In \emph{Conference On Learning Theory}, pp.\  3395--3398, 2018.

\bibitem[Jiang et~al.(2017)Jiang, Krishnamurthy, Agarwal, Langford, and
  Schapire]{jiang2017contextual}
Nan Jiang, Akshay Krishnamurthy, Alekh Agarwal, John Langford, and Robert~E
  Schapire.
\newblock Contextual decision processes with low {Bellman} rank are
  {PAC}-learnable.
\newblock In \emph{Proceedings of the 34th International Conference on Machine
  Learning}, pp.\  1704--1713, 2017.

\bibitem[Jin et~al.(2018)Jin, Allen-Zhu, Bubeck, and Jordan]{jin2018q}
Chi Jin, Zeyuan Allen-Zhu, Sebastien Bubeck, and Michael~I Jordan.
\newblock Is {Q}-learning provably efficient?
\newblock In \emph{Advances in Neural Information Processing Systems}, pp.\
  4863--4873, 2018.

\bibitem[Jin et~al.(2020{\natexlab{a}})Jin, Krishnamurthy, Simchowitz, and
  Yu]{jin2020reward}
Chi Jin, Akshay Krishnamurthy, Max Simchowitz, and Tiancheng Yu.
\newblock Reward-free exploration for reinforcement learning.
\newblock \emph{International Conference on Machine Learning},
  2020{\natexlab{a}}.

\bibitem[Jin et~al.(2020{\natexlab{b}})Jin, Yang, Wang, and
  Jordan]{jin2019provably}
Chi Jin, Zhuoran Yang, Zhaoran Wang, and Michael~I Jordan.
\newblock Provably efficient reinforcement learning with linear function
  approximation.
\newblock In \emph{Conference on Learning Theory}, pp.\  2137--2143,
  2020{\natexlab{b}}.

\bibitem[Jin et~al.(2020{\natexlab{c}})Jin, Yang, Wang, and
  Jordan]{jin2020provably}
Chi Jin, Zhuoran Yang, Zhaoran Wang, and Michael~I Jordan.
\newblock Provably efficient reinforcement learning with linear function
  approximation.
\newblock In \emph{Conference on Learning Theory}, pp.\  2137--2143. PMLR,
  2020{\natexlab{c}}.

\bibitem[Kakade(2003)]{kakade2003sample}
Sham~M Kakade.
\newblock \emph{On the sample complexity of reinforcement learning}.
\newblock PhD thesis, University of London London, England, 2003.

\bibitem[Kearns \& Singh(2002)Kearns and Singh]{kearns2002near}
Michael Kearns and Satinder Singh.
\newblock Near-optimal reinforcement learning in polynomial time.
\newblock \emph{Machine learning}, 49\penalty0 (2-3):\penalty0 209--232, 2002.

\bibitem[Kim et~al.(2022)Kim, Yang, and Jun]{kim2022improved}
Yeoneung Kim, Insoon Yang, and Kwang-Sung Jun.
\newblock Improved regret analysis for variance-adaptive linear bandits and
  horizon-free linear mixture mdps.
\newblock \emph{Advances in Neural Information Processing Systems},
  35:\penalty0 1060--1072, 2022.

\bibitem[Li et~al.(2021{\natexlab{a}})Li, Chen, Chi, Gu, and Wei]{li2021sample}
Gen Li, Yuxin Chen, Yuejie Chi, Yuantao Gu, and Yuting Wei.
\newblock Sample-efficient reinforcement learning is feasible for linearly
  realizable {MDP}s with limited revisiting.
\newblock \emph{Advances in Neural Information Processing Systems},
  34:\penalty0 16671--16685, 2021{\natexlab{a}}.

\bibitem[Li et~al.(2021{\natexlab{b}})Li, Shi, Chen, Gu, and
  Chi]{li2021breaking}
Gen Li, Laixi Shi, Yuxin Chen, Yuantao Gu, and Yuejie Chi.
\newblock Breaking the sample complexity barrier to regret-optimal model-free
  reinforcement learning.
\newblock \emph{Advances in Neural Information Processing Systems}, 34,
  2021{\natexlab{b}}.

\bibitem[Li et~al.(2023)Li, Yan, Chen, and Fan]{li2023minimax}
Gen Li, Yuling Yan, Yuxin Chen, and Jianqing Fan.
\newblock Minimax-optimal reward-agnostic exploration in reinforcement
  learning.
\newblock \emph{arXiv preprint arXiv:2304.07278}, 2023.

\bibitem[Li et~al.(2021{\natexlab{c}})Li, Wang, and Yang]{li2021settling}
Yuanzhi Li, Ruosong Wang, and Lin~F Yang.
\newblock Settling the horizon-dependence of sample complexity in reinforcement
  learning.
\newblock In \emph{IEEE Symposium on Foundations of Computer Science},
  2021{\natexlab{c}}.

\bibitem[Modi et~al.(2020)Modi, Jiang, Tewari, and Singh]{modi2020sample}
Aditya Modi, Nan Jiang, Ambuj Tewari, and Satinder Singh.
\newblock Sample complexity of reinforcement learning using linearly combined
  model ensembles.
\newblock In \emph{International Conference on Artificial Intelligence and
  Statistics}, pp.\  2010--2020. PMLR, 2020.

\bibitem[Sun et~al.(2019)Sun, Jiang, Krishnamurthy, Agarwal, and
  Langford]{sun2019model}
Wen Sun, Nan Jiang, Akshay Krishnamurthy, Alekh Agarwal, and John Langford.
\newblock {Model-based RL in contextual decision processes: PAC bounds and
  exponential improvements over model-free approaches}.
\newblock In \emph{Conference on Learning Theory}, pp.\  2898--2933, 2019.

\bibitem[Wang et~al.(2020)Wang, Du, Yang, and Kakade]{wang2020long}
Ruosong Wang, Simon~S Du, Lin~F Yang, and Sham~M Kakade.
\newblock Is long horizon reinforcement learning more difficult than short
  horizon reinforcement learning?
\newblock In \emph{Advances in Neural Information Processing Systems}, 2020.

\bibitem[Wang et~al.(2019)Wang, Wang, Du, and Krishnamurthy]{wang2019optimism}
Yining Wang, Ruosong Wang, Simon~S Du, and Akshay Krishnamurthy.
\newblock Optimism in reinforcement learning with generalized linear function
  approximation.
\newblock \emph{arXiv preprint arXiv:1912.04136}, 2019.

\bibitem[Weisz et~al.(2020)Weisz, Amortila, and
  Szepesv{\'a}ri]{weisz2020exponential}
Gellert Weisz, Philip Amortila, and Csaba Szepesv{\'a}ri.
\newblock Exponential lower bounds for planning in mdps with
  linearly-realizable optimal action-value functions.
\newblock \emph{arXiv preprint arXiv:2010.01374}, 2020.

\bibitem[Yang \& Wang(2019)Yang and Wang]{yang2019sample}
Lin Yang and Mengdi Wang.
\newblock Sample-optimal parametric {Q}-learning using linearly additive
  features.
\newblock In \emph{International Conference on Machine Learning}, pp.\
  6995--7004, 2019.

\bibitem[Zanette \& Brunskill(2019)Zanette and Brunskill]{zanette2019tighter}
Andrea Zanette and Emma Brunskill.
\newblock Tighter problem-dependent regret bounds in reinforcement learning
  without domain knowledge using value function bounds.
\newblock In \emph{International Conference on Machine Learning}, pp.\
  7304--7312, 2019.

\bibitem[Zanette et~al.(2020)Zanette, Lazaric, Kochenderfer, and
  Brunskill]{zanette2020learning}
Andrea Zanette, Alessandro Lazaric, Mykel Kochenderfer, and Emma Brunskill.
\newblock Learning near optimal policies with low inherent {Bellman} error.
\newblock In \emph{International Conference on Machine Learning}, 2020.

\bibitem[Zhang et~al.(2020)Zhang, Zhou, and Ji]{zhang2020almost}
Zihan Zhang, Yuan Zhou, and Xiangyang Ji.
\newblock Almost optimal model-free reinforcement learning via
  reference-advantage decomposition.
\newblock In \emph{Advances in Neural Information Processing Systems}, 2020.

\bibitem[Zhang et~al.(2021{\natexlab{a}})Zhang, Ji, and
  Du]{zhang2020reinforcement}
Zihan Zhang, Xiangyang Ji, and Simon Du.
\newblock Is reinforcement learning more difficult than bandits? a near-optimal
  algorithm escaping the curse of horizon.
\newblock In \emph{Conference on Learning Theory}, pp.\  4528--4531. PMLR,
  2021{\natexlab{a}}.

\bibitem[Zhang et~al.(2021{\natexlab{b}})Zhang, Yang, Ji, and
  Du]{zhang2021variance}
Zihan Zhang, Jiaqi Yang, Xiangyang Ji, and Simon~S Du.
\newblock Variance-aware confidence set: Variance-dependent bound for linear
  bandits and horizon-free bound for linear mixture mdp.
\newblock In \emph{Advances in Neural Information Processing Systems},
  2021{\natexlab{b}}.

\bibitem[Zhang et~al.(2022)Zhang, Ji, and Du]{zhang2022horizon}
Zihan Zhang, Xiangyang Ji, and Simon Du.
\newblock Horizon-free reinforcement learning in polynomial time: the power of
  stationary policies.
\newblock In \emph{Conference on Learning Theory}, pp.\  3858--3904. PMLR,
  2022.

\bibitem[Zhou \& Gu(2022)Zhou and Gu]{zhou2022computationally}
Dongruo Zhou and Quanquan Gu.
\newblock Computationally efficient horizon-free reinforcement learning for
  linear mixture mdps.
\newblock \emph{arXiv preprint arXiv:2205.11507}, 2022.

\bibitem[Zhou et~al.(2021)Zhou, Gu, and Szepesvari]{zhou2021nearly}
Dongruo Zhou, Quanquan Gu, and Csaba Szepesvari.
\newblock Nearly minimax optimal reinforcement learning for linear mixture
  markov decision processes.
\newblock In \emph{Conference on Learning Theory}, pp.\  4532--4576. PMLR,
  2021.

\end{thebibliography}
\bibliographystyle{iclr2024_conference}

\newpage
\appendix
%\appendixpage
\section{Technical Lemmas}\label{app:teclemma}
 \begin{lemma}[Theorem 4.3 in \cite{zhou2022computationally}]\label{lemma:take1}
Let $\{\mathcal{F}_i\}_{i=1}^{\infty}$ be a filtration, and $\{\psi_i,\zeta_i\}_{i\geq 1}$ be a stochastic process such that $\psi_i \in \mathbb{R}^{d}$ is $\mathcal{F}_i$-measurable and $\epsilon_i\in \mathbb{R}$ is $\mathcal{F}_{i+1}$ measurable. Let $L,\sigma, R, \lambda, c>0$, $\theta^{*}\in \mathbb{R}^d$. For $i\geq 1$, let $y_i=\left\langle \psi_i ,\theta^* \right \rangle+\zeta_i$ and suppose that 
\begin{align}
 \mathbb{E}[\zeta_i|\mathcal{F}_i]=0, \quad \mathbb{E}[\zeta_i^2|\mathcal{F}_i]\leq \sigma^2, |\zeta_i|\leq R, \|\psi_i\|_2 \leq L .\label{eq:4.4}
\end{align}
For $i\geq 1$, let $\Lambda_i= \lambda \mathbf{I} + \sum_{i'=1}^{i}\psi_{i'}\psi_{i'}^{\top}$ and  $b_i = \sum_{i'=1}^i y_{i'}\psi_{i'}$, $\theta_i = \Lambda_i^{-1}b_i$, and 
\begin{align}
\kappa_i   &  = 12\sqrt{\sigma^2 d\log(1+iL^2/(d\lambda))\log(32(\log(R/c)+1)i^2/\delta)} \nonumber\\ & + 24\log(32(\log(R/c)+1)k^2/\delta) \max_{1\leq i '\leq i}\{ |\zeta_{i'}| \min\{1, \|\psi_{i'}\|_{\Lambda_{i'-1}^{-1}}\} \}+ 6\log(32(\log(R/c)+1)k^2/\delta)c.\nonumber 
\end{align}
Then, for any $0<\delta<1$, with probability $1-\delta$.
\begin{align}
\forall i\geq 1,\quad  \left\|\sum_{i'=1}^i \psi_{i'}\zeta_{i'}\right\|_{\Lambda_{i}^{-1}}\leq \kappa_i, \quad \left\|\theta_i - \theta^*\right\|_{\Lambda_i}\leq \kappa_i + \sqrt{\lambda}\|\theta^*\|_2.\nonumber 
\end{align}
\end{lemma}

\begin{lemma}[ \cite{chen2021implicit}]\label{lemma:add1}
Let $X$ be a random variable taking value in $[-C,C]$ for some $C\geq 0$. Let $\mathrm{var}(Y)$ denote the variance of a random variable $Y$. It then holds that $\mathrm{var}(X^2)\leq 4C^2\mathrm{var}(X)$.
\end{lemma}

\begin{lemma}\label{lemma:tool2}[Lemma 10 in \cite{zhang2020almost}] \label{lemma:freeman_zhang}
Let $(M_{n})_{n\geq 0}$ be a  martingale  such that $M_0 = 0$  and $|M_{n}-M_{n-1}|\leq c$ for some $c>0$ and any $n\geq 1$. Let $\mathrm{Var}_{n}=\sum_{k=1}^{n}\mathbb{E}[(M_{k}-M_{k-1})^{2}|\mathcal{F}_{k-1}]$ for $n\geq 0$,
where $\mathcal{F}_{k}=\sigma(M_{1},M_{2},...,M_{k})$. Then for any positive integer $n$, and any $\epsilon,p>0$, we have that
\begin{equation}\label{self-bernstein}
\mathbb{P}\left[|M_{n}|\geq   2\sqrt{\mathrm{Var}_{n}\log(\frac{1}{p})}+2\sqrt{\epsilon\log(\frac{1}{p} )} +2c\log(\frac{1}{p}) \right] \leq \left(\frac{2nc^2}{\epsilon}+2\right)p.
\end{equation}
\end{lemma}

\begin{lemma}\label{lemma:dual} 
Let $\Phi$ be a bounded convex closed subset of $\mathbb{R}^d$. Let $\Phi^*= \{\psi\in \mathbb{R}^d | \phi^{\top}\psi\geq 0, \forall \phi \in \Phi \}$. Then there exists $\bar{\phi}\in \Phi$ satisfying that $\bar{\phi}^{\top}\psi\geq \frac{1}{2d}\max_{\phi\in \Phi}\phi^{\top}\psi$ for any $\psi \in \Phi^*$.
\end{lemma}
\begin{proof}  {Without loss of generality, we assume that $\|\psi\|_2 = 1$.}  
Let $u$ be the standard measure of $\mathbb{R}^d$. 
Define  {$\bar{\phi} = \frac{\int_{\Phi}\phi du(\phi)}{\int_{\Phi}d u(\phi)} $} be the geometry center of $\Phi$. We will show that 
$\bar{\phi}^{\top}\psi\geq \frac{1}{2d}\max_{\phi\in \Phi}\phi^{\top}\psi$ for any $\psi \in \Phi^*$. Fix $\psi \in \Phi^*$ and define $l = \max_{\phi\in \Phi}\phi^{\top}\psi$. 
 For $0\leq x\leq l$, we define $f(x) =\lim_{\epsilon\to 0} \frac{\int_{\phi\in \Phi, \phi^{\top}\psi \in [x,x+\epsilon)}du(\phi) }{\epsilon \cdot \int_{\phi}du(\phi) }$. Because $\Phi$ is a convex bounded set, $f(x)$ is well defined and continuous.  {By definition, it holds that $\int_{0}^{l}f(x)dx = 1$.}
 
 We claim that 
 \begin{align}
 \frac{f(l-y)}{y^{d-1}}\leq \frac{f(l-x)}{x^{d-1}} \label{eq:lrr}
 \end{align}
 for $0<x\leq y\leq l$.  

 {
\paragraph{Proof of \eqref{eq:lrr}.} 
Let $\mathcal{V}(w):= \{\phi\in \Phi: \phi^{\top}\psi=w  \}$ for $0\leq w\leq l$. $\mathcal{V}(l)$ is non-empty by definition. 
Fix $\tilde{\phi}\in \mathcal{V}(l)$. By convexity of $\Phi$, for any $\phi\in \mathcal{V}(l-y)$, $\frac{y-x}{y} \tilde{\phi}+ \frac{x}{y}\phi \in \mathcal{V}(l-x)$ for any $0<x\leq y \leq l$. As a result, $$\frac{y-x}{y}\tilde{\phi}+ \frac{x}{y}\mathcal{V}(l-y): = \left\{\frac{y-x}{y} \tilde{\phi}+ \frac{x}{y}\phi\mid\phi \in \mathcal{V}(l-y)\right\}\subset \mathcal{V}(l-x).$$
Let $\tilde{u}$ be the standard measure of $\mathbb{R}^{d-1}$. It then holds that
\begin{align}
\tilde{u}(\mathcal{V}(l-x))\geq \tilde{u}(\mathcal{V}(l-y))\cdot \left( \frac{x}{y}\right)^{d-1}
\end{align} 
for $0<x\leq y\leq l$. 
By definition of $f$, we have that
\begin{align*}
&\lim_{\epsilon\to 0} \frac{\epsilon \min_{w'\in [w,w+\epsilon)}\tilde{u}(\mathcal{V}(w'))}{\epsilon\int_{\Phi}du(\phi)}\\
\leq & f(w): =\lim_{\epsilon\to 0} \frac{\|\psi\|_2\int_{\phi\in \Phi, \phi^{\top}\psi\in [w,w+\epsilon) }du(\phi)}{\epsilon \int_{\Phi}du(\phi)} \\ = &\lim_{\epsilon\to 0} \frac{\epsilon\max_{w'\in [w,w+\epsilon)}\tilde{u}(\mathcal{V}(w')) }{\epsilon\int_{\Phi}du(\phi)}.
\end{align*}
In then follows that
\begin{align}
\lim_{\epsilon\to 0} \frac{ \min_{w'\in [w,w+\epsilon)}\tilde{u}(\mathcal{V}(w'))}{\int_{\Phi}du(\phi)} \leq f(w)\leq \lim_{\epsilon\to 0} \frac{ \max_{w'\in [w,w+\epsilon)}\tilde{u}(\mathcal{V}(w'))}{\int_{\Phi}du(\phi)}.\nonumber
\end{align}
As a result, we have that 
\begin{align}
 \frac{f(l-y)}{y^{d-1}} & \leq \lim_{\epsilon\to 0}\max_{\upsilon\in [0,\epsilon)]} \frac{\tilde{u}(\mathcal{V}(l-y+\upsilon))}{y^{d-1} \int_{\Phi}du(\phi) }\nonumber
\\ &  \leq \lim_{\epsilon\to 0}\max_{\upsilon\in [0,\epsilon)} \min_{\upsilon'\in [0,\epsilon) }\frac{\tilde{u}(\mathcal{V}(l-x+\upsilon'))}{x^{d-1} \int_{\Phi}du(\phi)} \cdot \left(\frac{(y-\upsilon)x}{y(x-\upsilon')} \right)^{d-1}\nonumber
\\ & \leq  \frac{f(l-x)}{x^{d-1}}. \nonumber
\end{align}
The proof of \eqref{eq:lrr} is finished.
}

 Let $z = \bar{\phi}^{\top}\psi$ and $f(z)=b$,
 we then have that $f(x)\geq \frac{b\cdot (l-x)^{d-1}}{ (l-z)^{d-1}}$ for any $x\in [z,l]$, and $f(x)\leq \frac{b(l-x)^{d-1}}{(l-z)^{d-1}}$ for any $x\in [z]$. 
 
 %\begin{align}
%\int_{0}^{l/2d} f(x)dx\leq \int_{l/2d}^l f(x)dx.
% \end{align}
 {
By definition, we have that
\begin{align}
z &= \frac{\int_{\Phi}\phi^{\top}\psi du(\phi)}{\int_{\phi}du(\phi)} \nonumber
\\ & = \lim_{\epsilon\to 0}\sum_{i=0}^{\left\lceil l/\epsilon\right\rceil}\frac{\int_{\phi\in \Phi, \phi^{\top}\psi \in [i\epsilon,(i+1)\epsilon) } \phi^{\top}\psi du(\phi) }{\int_{\Phi}du(\phi)}\nonumber
\\ &  = \lim_{\epsilon\to 0}\sum_{i=0}^{\left\lceil l/\epsilon\right\rceil}\frac{\int_{\phi\in \Phi, \phi^{\top}\psi \in [i\epsilon,(i+1)\epsilon) } i\epsilon du(\phi) }{\int_{\Phi}du(\phi)}\nonumber
\\ & = \lim_{\epsilon\to 0}\sum_{i=0}^{\left\lceil l/\epsilon\right\rceil}f(i\epsilon)i\epsilon^2\nonumber
\\ & = \int_{0}^l f(x)xdx. \nonumber
\end{align}
}

As a result, we have that
\begin{align}
\int_{ 0\leq x \leq z}f(x)\cdot (z-x) d(x)= \int_{ z\leq x\leq l} f(x)(x-z)dx, \nonumber
\end{align}
which implies
\begin{align}
 & \int_{z\leq x\leq l} \frac{b(l-x)^{d-1}(x-z)}{(l-z)^{d-1}}\leq \int_{ z\leq x\leq l} f(x)(x-z)dx  =\int_{ 0\leq x \leq z}f(x)\cdot (z-x) d(x) \leq \int_{0\leq x\leq z} \frac{b(l-x)^{d-1} (z-x)}{(l-z)^{d-1}}. \nonumber
\end{align}
Therefore, $\int_{z\leq x\leq l}(l-x)^{d-1}(x-z)dx\leq \int_{0\leq x\leq z} (l-x)^{d-1}(z-x)dx$, which means 
\begin{align}
\frac{(l-z)^{d+1}}{d(d+1)}\leq \frac{l^{d+1}}{d+1}- \frac{(l-z)l^{d}}{d} + \frac{(l-z)^{d+1}}{d(d+1)}. 
\end{align}

Then it holds that $z\geq \frac{l}{d+1}\geq \frac{l}{2d}$. The proof is completed.
\end{proof}

\begin{lemma}\label{lemma:tool5} Let $l$ be a positive integer. 
Let $\Phi = \{\phi_i\}_{i=1}^n$ and $\Psi = \{\psi_j\}_{j=1}^m$ be two group of vectors in $\mathbb{R}^l$ satisfying that $\phi_i^{\top}\psi_j\geq 0$ for any $1\leq i \leq n$ and $1\leq j \leq m$. It then holds that
\begin{align}
\sum_{i=1}^n \max_j \phi_i^{\top}\psi_j\leq 2l\max_{j}\sum_{i=1}^n \phi_i^{\top}\psi_j .\nonumber
\end{align}
\end{lemma}
\begin{proof}
By Lemma~\ref{lemma:dual}, there exists $\psi^*\in \mathrm{Conv}(\Psi)$ such that $\phi_i^{\top}\psi^*\geq \frac{1}{2l}\max_j \phi_i^{\top}\psi_j$ for any $1\leq i \leq n$. As a result, we have that
\begin{align}
 {
\sum_{i=1}^n \max_j \phi_i^{\top}\psi_j\leq 2l\sum_{i=1}^n \phi_i^{\top} \psi^*\leq 2 l\max_{j}\sum_{i=1}^n \phi_i^{\top}\psi_j.\nonumber}
\end{align}
The proof is completed.
\end{proof}

\begin{lemma}\label{lemma:epl}
Let $\{\phi_i\}_{i= 1}^n$ be a group of vectors in $\mathbb{R}^d$ such that $\|\phi_i\|_2 \leq L$. Fix $\lambda>0$ and let $\Lambda_i  =\lambda\mathbf{I}+\sum_{i'=1}^{i}\phi_i\phi_i^{\top}$. For any sequence $0= i_1<i_2<\ldots<i_k= n$, 
\begin{align}
\sum_{j=1}^k  \min\left\{ \sum_{i=i_j+1}^{i_{j+1}} \phi_i^{\top} \Lambda_{i_j}^{-1}\phi_i,1 \right\} \leq 6d\log(nL/\lambda).
\end{align}
\end{lemma}
\begin{proof}
Let $\mathcal{J}\subset [k-1]$ be the indices $j$ such that $\Lambda_{i_{j+1}} \preccurlyeq 2\Lambda_{i_j}$ does not hold. Then we have
\begin{align}
2^{|\mathcal{J}|}\leq \Pi_{j\in \mathcal{J}}\frac{\mathrm{det}(\Lambda_{i_{j+1}})}{\mathrm{det}(\Lambda_{i_j})} \leq (nL^2)^d,
\end{align}
which implies $|\mathcal{J}|\leq 2d\log_2(nL/\lambda)$.

Continue the computation, 
\begin{align}
\sum_{j=1}^k  \min\left\{ \sum_{i=i_j+1}^{i_{j+1}} \phi_i^{\top} \Lambda_{i_j}^{-1}\phi_i,1 \right\} &  \leq |\mathcal{J}|+\sum_{j\notin \mathcal{J}}\sum_{i=i_j+1}^{i_{j+1}} \phi_i^{\top} \Lambda_{i_j}^{-1}\phi_i \nonumber
\\ & \leq 2d\log_2(nL/\lambda) + 2\sum_{i=1}^n \phi_i^{\top}\Lambda_{i+1}^{-1}\phi_i \nonumber
\\ & \leq 6d\log_2(nL/\lambda).
\end{align}
The proof is completed.
\end{proof}

\section{Regret Analysis (Proof of Theorem~\ref{thm:main})}\label{app:reg}

%\begin{lemma}\label{lemma:conti}
%Let the dataset $\mathcal{D}:= \{ (s_i,a_i,s_i')\}_{i=1}^n$ be fixed.
%\end{lemma}

In this section, we present regret analysis for Algorithm~\ref{alg:main}, i.e., the proof of Theorem~\ref{thm:main}.

%The left of this section is devoted to the proof of Theorem~\ref{thm:main}.

%We need the key lemma below to prove Theorem~\ref{thm:main}. And the rest of this section is devoted to the proof of Lemma~\ref{lemma:key}.
%\begin{condition}\label{cond1} 
%For any proper $h,k$, it holds that
%\begin{align}
%\max_{\phi \in \mathrm{Conv}(\Phi)}\sqrt{\phi^{\top} \left( \lambda\mathbf{I} + \sum_{i = 1}^n\phi(s_i,a_i)(\phi(s_i,a_I ))^{\top}\right)^{-1}\phi}\leq \frac{1}{64d\alpha}.
%\end{align}
%\end{condition}

%\begin{lemma}\label{lemma:key}
%Assume Condition~\ref{cond1} holds for dataset $\mathcal{D}$. By taking $\mathcal{D}$ as the initial data set, the regret of Algorithm~\ref{alg:main} in $k$ episodes is bounded by  $\tilde{O}(\sqrt{d^{4.5}\sqrt{K}}+d^9)$ with probability $1-\delta$.
%\end{lemma}

%Given Lemma~\ref{lemma:key}, we prove Theorem~\ref{thm:main} as below.
%\begin{proof}[Proof of Theorem~\ref{thm:main}]
%Note that there are at most $j_{\mathrm{max}}+1 = O(d\log(KH))$ outer epochs. Let $1\leq j \leq j_{\mathrm{max}}$ be fixed. For the $j$-th outer epoch for $1\leq$, since we set $r^{(j)}(s,a) = \mathbb{I}\left[ 
 %  ..\geq ..  \right]$ and re-direct all state-action pairs with non-zero reward to a dumb state, then Condition~\ref{cond1} holds for the modified MDP

%\end{proof}

\subsection{The Successful Event}

%\simon{give some intuitions about the $\epsilon$-net in the linear MDP context.}

We first introduce the successful event $\mathcal{G}$. 
Fix $k\in [K]$. For any $v\in \mathcal{W}_{\epsilon}$, let $(\hat{\theta}^k(v) , \tilde{\theta}^k(v), \Lambda^k(v) )$ be the output of Algorithm~\ref{alg:vlambdae} with input as $\{s_{h}^{k'},a_{h}^{k'},s_{h+1}^{k'} \}_{k'\in [k-1],h\in [H]}$ and $v$. Recall that $\theta(v)=\mu^{\top}v$. Let $\kappa    = 13\sqrt{6d^2\log^2(KH/\delta)} +72\log(KH/\delta)$. Define  $\mathcal{G}^k (v)$ to be the event where
\begin{align}
 \| \theta(v) -\hat{\theta}(v) \|_{\Lambda^k(v)} \leq \kappa,  \quad \|\theta(v^2) - \tilde{\theta}(v)\|_{\Lambda^k(v)} \leq 4\kappa. 
\end{align}

With Lemma~\ref{lemma:tool1}, we have that $\mathrm{Pr}(\mathcal{G}^k (v))\geq 1-10KH\delta/|\mathcal{W}_{\epsilon}|$.

Define $\mathcal{G} ^k = \cap_{v\in \mathcal{W}_{\epsilon}}\mathcal{G} ^k(v)$ and $\mathcal{G}_{\mu}  = \cap_k \mathcal{G} ^k$. Then we have that $\mathrm{Pr}(\mathcal{G}_{\mu} ) \geq 1-10K^2H\delta$. On the other hand, we define $\mathcal{G}_{r} = \{ \theta_r \in \Theta^k ,\forall 1\leq k \leq K \}$. By Lemma~\ref{assum1}, we have that $\mathrm{Pr}(\mathcal{G}_r)\geq 1- 10KH\delta$. Let $\mathcal{G}= \mathcal{G}_{\mu}\cap\mathcal{G}_r$. It then holds that $\mathrm{Pr}(\mathcal{G})\geq 1-20KH\delta$. 
In the rest of this section, we continue the proof conditioned on  $\mathcal{G}$.

\subsection{Regret Decomposition}

We start with showing that the maintained value function and $Q$-function are nearly optimistic.

\begin{lemma}\label{lemma:opt}
Conditioned on $\mathcal{G}$, it holds that $\theta_{r}\in \Theta^k$ and $\mu\in \mathcal{U}^k$ for any $k$.
\end{lemma}
\begin{proof}
Recall that
\begin{align}
\mathcal{U}^k = \{ \tilde{\mu}\in \mathcal{U} |  |\phi^{\top}\tilde{\mu}^{\top}v -\phi^{\top}\hat{\theta}^k(v) | \leq b^k(v,\phi), \forall \phi\in \Phi(\epsilon), v\in \mathcal{W}_{\epsilon}    \} \nonumber
\end{align}
and $b^k(v,\phi) = \alpha \|\phi\|_{(\Lambda^k(v))^{-1}} +4\epsilon$. By the definition of $\mathcal{G}$, we have that for any $\phi\in \Phi$ and $v\in \mathcal{W}_{\epsilon}$,
\begin{align}
|\phi^{\top}\mu^\top v - \phi^{\top}\hat{\theta}^k(v)|\leq  \|\phi\|_{(\Lambda^k(v))^{-1}}\cdot \| \theta(v)-\hat{\theta}^k(v) \|_{\Lambda^k(v)} \leq \kappa\|\phi\|_{(\Lambda^k(v))^{-1}} \leq b^k(v,\phi). \nonumber
\end{align}
As a result, $\mu\in \mathcal{U}^k$. On the other hand,  $\theta_{r}\in \Theta^k$ by the definition of $\mathcal{G}$.
\end{proof}

For any proper $s,h,k$, we define $V_h^k(s): =\mathbb{E}_{\pi^k}[  \sum_{h'=h}^H r_{h'}|s_h = s, \tilde{\mu}^k,\theta^k]$ to be the value function w.r.t. the model $(\mu^k,\theta^k)$.

By definition of regret,  we have that
\begin{align}
&   \mathrm{Regret(K)} \nonumber
\\ & = \sum_{k=1}^K \min\left\{ \left(V^*_1(s_1^k) - V^{\pi^k}_1(s_1^k) \right)  ,1 \right\}\nonumber
\\ &  \leq  \sum_{k=1}^K \min\left\{ \left(V^k_1(s_1^k) - V^{\pi^k}_1(s_1^k) \right),1 \right\} \nonumber
\\  &\leq  \sum_{k=1}^K \underbrace{\min\left\{ \sum_{h=1}^H \left( (\phi_h^k)^{\top}\theta^k - (\phi_h^k)^{\top}\theta_r +  (\phi_{h}^k)^{\top}(\mu^k)^{\top}V_{h+1}^k-          (\phi_h^k)^{\top}\mu^{\top}V_{h+1}^k\right) ,1\right\} }_{T_1(k)} \nonumber
\\ & \qquad \qquad +\sum_{k=1}^K \underbrace{  \min \left\{\sum_{h=1}^H \left((\phi_h^k)^{\top}\mu^{\top}V_{h+1}^k - V_{h+1}^k(s_{h+1}^k) \right) ,1\right\} }_{T_2(k)}  \nonumber\\ & \qquad \qquad +\sum_{k=1}^K\underbrace{\left(\sum_{h=1}^H r_h^k - V_{1}^{\pi^k}(s_1^k) \right) }_{T_3(k)}\label{eq:decom}
\end{align}
Here the first inequality holds by the optimality of $(\tilde{\mu}^k,\theta^k)$. The right hand side of \eqref{eq:decom} consists of three terms, where $\sum_k T_1(k)$ is  the error due to inaccurate transition and reward model, $\sum_{k}T_2(k)$ is the martingale difference due to state transition, and $\sum_{k}T_3(k)$ is the difference between the expected accumulative reward and  the empirical accumulative reward. We have the lemma below to bound these terms.

\begin{lemma}\label{lemma:bdt1}
Conditioned on $\mathcal{G}$, with probability $1-10KH\delta$, it holds that
\begin{align}
\sum_{k=1}^K T_1(k)\leq \tilde{O}(d^{5.5}\sqrt{K}+d^{6.5}).\nonumber
\end{align}
\end{lemma}
%\simon{add some interpretations of each term.}

\begin{lemma}\label{lemma:bdt23}
Conditioned on $\mathcal{G}$, with probability $1-10KH\delta$, it holds that
\begin{align}
\sum_{k=1}^K (T_2(k)+T_3(k))\leq 8\sqrt{K\iota} +21\iota.
\end{align}
\end{lemma}

The proofs of Lemma~\ref{lemma:bdt1} and Lemma~\ref{lemma:bdt23} are presented in Appendix~\ref{sec:pft1} and Appendix~\ref{sec:pft23} respectively.

\subsection{Putting All Pieces Together}

By \eqref{eq:decom}, Lemma~\ref{lemma:bdt1} and \ref{lemma:bdt23}, we conclude that, with probability $1-50K^2H^2\delta$, it holds that $\mathrm{Regret}(K) =  \tilde{O}(d^{5.5}\sqrt{ K}+d^{6.5})$.  The proof is completed by replacing $\delta$ with $\frac{\delta}{50K^2H^2}$.

%It is worth noting that the regret due to the unknown transition dynamic is bounded by $\tilde{O}(d^2\sqrt{K})$, while the regret due to unknown reward is much larger

\subsection{Missing Proofs}
\subsubsection{Bound of Term $T_1(k)$ (Proof of Lemma~\ref{lemma:bdt1})}\label{sec:pft1}

Direct computation gives that
\begin{align}Í
\sum_{k=1}^K T_1(k) &  = \sum_{k=1}^K \min\left\{ \sum_{h=1}^H \left( (\phi_h^k)^{\top}\theta^k - (\phi_h^k)^{\top}\theta_r +  (\phi_{h}^k)^{\top}(\tilde{\mu}^k)^{\top}V_{h+1}^k-          (\phi_h^k)^{\top}\mu^{\top}V_{h+1}^k\right) ,1\right\} \nonumber
\\ & \leq \sum_{k=1}^K \min\left\{\sum_{h=1}^H ( (\phi_h^k)^{\top} \theta^k - (\phi_h^k)^{\top}\theta_r  ),1\right\} + \sum_{k=1}^K \min\left\{\sum_{h=1}^H  ((\phi_{h}^k)^{\top}(\tilde{\mu}^k)^{\top}V_{h+1}^k-          (\phi_h^k)^{\top}\mu^{\top}V_{h+1}^k),1 \right\}.\label{eq:mark1}
\end{align}

By Lemma~\ref{assum1}, with probability $1-\delta$,
\begin{align}
 \sum_{k=1}^K \min\left\{\sum_{h=1}^H ( (\phi_h^k)^{\top} \theta^k - (\phi_h^k)^{\top}\theta_r  ),1\right\} &  = \tilde{O}\left(d^{6}\sqrt{\sum_{k=1}^K \sum_{h=1}^H \mathrm{Var}(R(s_h^k,a_h^k)/\sqrt{d})} + d^{6.5}\right)  \nonumber
 \\ &  = \tilde{O}(d^{5.5}\sqrt{K}+d^{6.5}).\label{eq:boundr}
\end{align}
Here  we use the fact that $\sum_{h=1}^H \mathrm{Var}(R(s_h^k,a_h^k))\leq \sum_{h=1}^H \bar{R}_h^k \leq 1$ with $\bar{R}_h^k$ as the maximal possible value of $R(s_h^k,a_h^k)$.

As for the second term in \eqref{eq:mark1}, noting that $V_{h+1}^k, \in \mathcal{W}$ for any proper $k,h$, letting $\bar{V}_{h+1}^k\in {\mathcal{W}_{\epsilon}}$ be such that $\|\bar{V}_{h+1}^k -V_{h+1}^k\|_{\infty}\leq \epsilon$, we have that
\begin{align}
&\sum_{k=1}^K \min\left\{\sum_{h=1}^H  ((\phi_{h}^k)^{\top}(\tilde{\mu}^k)^{\top}V_{h+1}^k-          (\phi_h^k)^{\top}\mu^{\top}V_{h+1}^k),1 \right\} \nonumber
\\ & \leq \sum_{k=1}^K \min\left\{  \sum_{h=1}^H  ((\phi_{h}^k)^{\top}(\tilde{\mu}^k)^{\top}\bar{V}_{h+1}^k-          (\phi_h^k)^{\top}\mu^{\top}\bar{V}_{h+1}^k)  ,1\right\} +2KH\epsilon\nonumber
\\ & \leq \sum_{k=1}^K \min\left\{  \sum_{h=1}^H  b^k(\bar{V}_{h+1}^k,\phi_h^k)  ,1\right\} +2KH\epsilon\nonumber
\\ & \leq \alpha  \sum_{k=1}^K \min\left\{  \sum_{h=1}^H  \sqrt{ (\phi_h^k)^{\top} (\Lambda^k(\bar{V}_{h+1}^k))^{-1}\phi_h^k  } ,1\right\}   + 6KH\epsilon.\label{eq:exp2}
\end{align}

Define $\beta_h^k =\sqrt{ (\phi_h^k)^{\top} (\Lambda^k(\bar{V}_{h+1}^k))^{-1}\phi_h^k  } $.

For $v\in \mathbb{R}^S$, let $( \hat{\theta}(v),\tilde{\theta}(v) ,\Lambda^k(v)$ be the output of Algorithm~\ref{alg:vlambdae} with input as $\{s_{h'}^{k'},a_{h'}^{k'},s_{h'+1}^k,  \}_{k'\in [k-1], h' \in [H]}$ and $v$. Define 
\begin{align}
(\sigma_h^k(v))^2 = (\phi_h^k)^{\top}\tilde{\theta}(v) - \left( (\phi_h^k)^{\top}\hat{\theta}(v)\right)^2  + 16\alpha \sqrt{(\phi_h^k)^{\top}(\Lambda^k(v))^{-1}} \phi_h^k  + 4 \epsilon.\label{def:gen_variance}
\end{align}
In words, $(\sigma_h^k(v))^2$ is the estimator for the variance $\mathbb{V}(P_{s_h^k,a_h^k},v)$. By the definition of $\mathcal{G}$ and the fact that $\bar{V}_{h'}^{k'} ,(\bar{V}_{h'}^{k'} )^2\in {\mathcal{W}_{\epsilon}}$,  we have that $\sigma_h^k(\bar{V}_{h'}^{k'})\geq  {\mathbb{V}(P_{s_h^k,a_h^k},\bar{V}_{h'}^{k'})}$.
%We then can rewrite $\Lambda_h^k $ as $\Lambda_h^k =\Gamma_h^k(V_{h+1}^k) = \lambda \mathbf{I} + \sum_{k'=1}^{k-1}\sum_{h=1}^H\frac{\phi_{h'}^{k'} (\phi_{h'}^{k'})^{\top}}{(\sigma_{h'}^{k'}(v^k_{h+1}))^2}$.
Let $i_{\mathrm{max}}=\log_2(H)+1$.
Recall $\mathcal{H}_{i} = \{h| H - \frac{H}{2^{i-1}}+1\leq h\leq   H - \frac{H}{2^i}\}$ for $i=1,2,\ldots, i_{\mathrm{max}}-1$and $\mathcal{H}_{i_{\mathrm{max}}}=\{H\}$.  Then we let
 $\mathcal{V}_{i} = \{\bar{V}_{h+1}^k| 1\leq k\leq K, 
h \in \mathcal{H}_i\}$ for $i=1,2,\ldots, \log_2(H)$ and $\mathcal{V}_{i_{\mathrm{max}}} = \{ \bar{V}_{H+1}^k | 1\leq k \leq K\}$. 

Fix $i$. For $h \in \mathcal{H}_i$,  we define $\bar{\Lambda_{(i)}^k}$ as
\begin{align}
\bar{\Lambda}^k_{(i)} = \lambda \mathbf{I} + \sum_{k'=1}^{k-1}\sum_{h\in \mathcal{H}_i} \frac{\phi_{h'}^{k'} (\phi_{h'}^{k'})^{\top}}{\max_{v\in \mathcal{V}_{i}}(\sigma_{h'}^{k'}(v))^2}.
\end{align}
By definition, it holds that $ {\bar{\Lambda}^k_{(i)}}\preccurlyeq \Lambda^k(\bar{V}_{h+1}^k
)$ for any $k$ and $h'\in \mathcal{H}_i$. It then holds that
\begin{align}
 & \sum_{k=1}^K\sum_{i=1}^{i_{\mathrm{max}}} \min \left\{ \sum_{h \in \mathcal{H}_i}\beta_h^k ,1 \right\} \nonumber
 \\ &   = \sum_{k=1}^K\sum_{i=1}^{i_{\mathrm{max}}}  \min\left\{ \sum_{h\in \mathcal{H}_i} \sqrt{ (\phi_h^k)^{\top} (\Lambda^k(\bar{V}_{h+1}^k))^{-1}\phi_h^k  } , 1\right\} .\label{eq:eee3}
%\label{eq:betabound1}
\end{align}

Noting that $ {\bar{\Lambda}^k_{(i)} + \sum_{h\in \mathcal{H}_i}\frac{\phi_{h}^{k} (\phi_{h}^{k})^{\top}}{\max_{v\in \mathcal{V}_i}(\sigma_{h}^{k}(v))^2}  =\bar{\Lambda}^{k+1}_{(i)}  } $, by Lemma~\ref{lemma:epl}, we have that
\begin{align}
& \sum_{k=1}^K \min\left\{ \sum_{h\in \mathcal{H}_i}\frac{(\phi_h^k)^{\top} ( {\bar{\Lambda}^k_{(i)})^{-1}}\phi_h^k}{\max_{v\in \mathcal{V}_i} (\sigma_{h}^k(v))^2 }  ,1\right\}  \leq 16d\log(KH),\label{eq:betabound1}
\\ & \sum_{k=1}^K I^k_i \leq 16d\log(KH).\label{eq:betabound2}
\end{align}
where $I^k_i =\mathbb{I}\left[ \sum_{h\in \mathcal{H}_i}\frac{(\phi_h^k)^{\top} ( {\bar{\Lambda}^k_{(i)}})^{-1}\phi_h^k}{\max_{v\in \mathcal{V}_i} (\sigma_{h}^k(v))^2 }> 1 \right] $. 
%and $\check{I}^k_i =\mathbb{I}\left[\sum_{h\in \mathcal{H}_i} \sqrt{(\phi_h^k)^{\top} (\bar{\Lambda}^k)^{-1}\phi_h^k}\geq 1\right] $. 

By Lemma~\ref{lemma:epl}, we have
\begin{align}
\sum_{k=1}^K (1-I^k_i) \sum_{h\in \mathcal{H}_i}\sqrt{(\phi_h^k)^{\top} ( {\bar{\Lambda}^k_{(i)}})^{-1}\phi_h^k} &  \leq \sqrt{\sum_{k=1}^K \sum_{h\in \mathcal{H}_i} \frac{(\phi_h^k)^{\top} ( {\bar{\Lambda}^k_{(i)}})^{-1}\phi_h^k}{\max_{v\in \mathcal{V}_i} (\sigma_{h}^k(v))^2 }  } \cdot \sqrt{\sum_{k=1}^K (1-I_i^k)\sum_{h\in \mathcal{H}_i} \max_{v\in \mathcal{V}_i}(\sigma_h^k(v))^2  } \nonumber
\\ & \leq 16d\log(KH)\cdot \sqrt{\sum_{k=1}^K (1-I^k_i)\sum_{h\in \mathcal{H}_i} \max_{v\in \mathcal{V}_i}(\sigma_h^k(v))^2 }.\label{eq:ff1}
\end{align}
 By the successful event $\mathcal{G}$, and noting that $v\in {\mathcal{W}_{\epsilon}}$, we have that 
\begin{align}
(\sigma_h^k(v))^2 &  \leq \mathbb{V}(P_{s_h^k,a_h^k},v) + (\phi_h^k)^{\top} (\tilde{\theta}(v)-\theta(v^2))  - 2 |(\phi_h^k)^{\top}(\hat{\theta}(v) - \theta(v) )| + 16\alpha \sqrt{(\phi_h^k)^{\top} (\Lambda^k(v))^{-1}\phi_h^k } + 4\epsilon \nonumber
\\ & \leq  \mathbb{V}(P_{s_h^k,a_h^k},v) +(6\kappa +16\alpha)\sqrt{(\phi_h^k)^{\top} (\Lambda^k(v))^{-1}\phi_h^k } + \frac{1}{KH}
\\ & \leq  \mathbb{V}(P_{s_h^k,a_h^k},v)+(6\kappa +16\alpha)\sqrt{(\phi_h^k)^{\top} (\bar{\Lambda}^k)^{-1}\phi_h^k } + \frac{1}{KH}\nonumber
\\ & \leq  \mathbb{V}(P_{s_h^k,a_h^k},v)+32\alpha \sqrt{(\phi_h^k)^{\top} (\bar{\Lambda}^k)^{-1}\phi_h^k } + \frac{1}{KH}.\nonumber 
\end{align}

Therefore,
\begin{align}\sum_{k=1}^K  (1-I_i^k)\sum_{h\in \mathcal{H}_i} \max_{v\in \mathcal{V}_i} (\sigma_h^k(v))^2 \leq \sum_{k=1}^K \sum_{h\in \mathcal{H}_i}\max_{v\in \mathcal{V}_i}\mathbb{V}(P_{s_h^k,a_h^k},v)+32\alpha \sum_{k=1}^K (1-I_i^k)\sum_{h\in \mathcal{H}_i}\sqrt{ (\phi_h^k)^{\top} ( {\bar{\Lambda}^k_{(i)}})^{-1}\phi_h^k }+1.\label{eq:ff2}
\end{align}

By \eqref{eq:ff1} and \eqref{eq:ff2}, we obtain that
\begin{align}
 & \sum_{k=1}^K (1-I^k_i) \sum_{h\in \mathcal{H}_i}\sqrt{(\phi_h^k)^{\top} ( {\bar{\Lambda}^k_{(i)}})^{-1}\phi_h^k} \nonumber
 \\ & \leq  16d\log(KH)\cdot \sqrt{ \sum_{k=1}^K \sum_{h\in \mathcal{H}_i}\max_{v\in \mathcal{V}_i}\mathbb{V}(P_{s_h^k,a_h^k},v)+32\alpha \sum_{k=1}^K (1-I_i^k)\sum_{h\in \mathcal{H}_i}\sqrt{ (\phi_h^k)^{\top} ( {\bar{\Lambda}^k_{(i)}})^{-1}\phi_h^k }+1  },
\end{align}
which implies that 
\begin{align}
\sum_{k=1}^K (1-I^k_i) \sum_{h\in \mathcal{H}_i}\sqrt{(\phi_h^k)^{\top} ( {\bar{\Lambda}^k_{(i)}})^{-1}\phi_h^k} \leq 32\log(KH) \sqrt{\sum_{k=1}^K \sum_{h\in \mathcal{H}_i }\max_{v\in \mathcal{V}_i}\mathbb{V}(P_{s_h^k,a_h^k},v)+1} +25000\log^2(KH)\alpha.\label{eq:ff3} 
\end{align}

Using \eqref{eq:eee3}, \eqref{eq:betabound1}, \eqref{eq:betabound2} and \eqref{eq:ff3}, we learn that
\begin{align}
 & \sum_{k=1}^K \sum_{i=1}^{i_{\mathrm{max}}}\min\left\{\sum_{h\in \mathcal{H}_i} \beta_h^k,1 \right\} \nonumber
 \\& = \sum_{k=1}^K\sum_{i=1}^{i_{\mathrm{max}}}  \min\left\{ \sum_{h\in \mathcal{H}_i} \sqrt{ (\phi_h^k)^{\top} (\Lambda^k(\bar{V}_{h+1}^k))^{-1}\phi_h^k  } , 1\right\} \nonumber
 \\ & \leq \sum_{i=1}^{i_{\mathrm{max}}}\sum_{k=1}^K (1-I^k_i) \sum_{h\in \mathcal{H}_i}\sqrt{(\phi_h^k)^{\top} ( {\bar{\Lambda}^k_{(i)}})^{-1}\phi_h^k} +\sum_{i=1}^{i_{\mathrm{max}}} \sum_{k=1}^K I_i^k 
 \\ &\leq i_{\mathrm{max}}\gamma+\sum_{i=1}^{i_{\mathrm{max}}}\left(   2\gamma\sqrt{\sum_{k=1}^K\sum_{h\in \mathcal{H}_i}\max_{v\in \mathcal{V}_i}\mathbb{V}(P_{s_h^k,a_h^k},v) +1} +100\gamma^2\alpha+ {16KH\epsilon}\right)\label{eq:ff4}
\end{align}
with $\gamma = 16d\log(KH)$.

%Combining \eqref{eq:betabound1} and \eqref{eq:betabound2}, we obtain that
%\begin{align}
%\sum_{k=1}^K \sum_{i=1}^{i_{\mathrm{max}}} \min \left\{ \sum_{h\in \mathcal{H}_i}\beta_h^k ,1 \right\} \leq  i_{\mathrm{max}}\max_i\sqrt{\gamma}\sqrt{\sum_{k=1}^K\mathbb{I}\left[ \sum_{h\in \mathcal{H}_i}\frac{(\phi_h^k)^{\top} (\bar{\Lambda}^k)^{-1}\phi_h^k}{\max_{v\in \mathcal{V}_i} (\sigma_{h}^k(v))^2 }\leq  1 \right]\sum_{h\in \mathcal{H}_i}\max_{v\in \mathcal{V}_i} (\sigma_{h}^k(v))^2} + \gamma,\label{eq:betabound3}
%\end{align}

Below we fix $i$ and bound the total variance term $\sum_{k=1}^K\sum_{h\in \mathcal{H}_i}\max_{v\in \mathcal{V}_i}\mathbb{V}(P_{s_h^k,a_h^k},v)$. Let $v\in \mathcal{V}_i$ be fixed.
%Recall that the definition of $$.

%Define $I^k = {I}\left[ \sum_{h\in \mathcal{H}_i}\frac{(\phi_h^k)^{\top} (\bar{\Lambda}^k)^{-1}\phi_h^k}{\max_{v\in \mathcal{V}_i} (\sigma_{h}^k(v))^2 }\leq  1 \right].$
%Fixing $v\in \mathcal{V}_i$,  and taking sum over $h\in \mathcal{H}_i$, 
% $k\in [K]$,  we obtain that 
%\begin{align}
%\sum_{k=1}^K I^k\sum_{h\in \mathcal{H}_i} \sigma_h^k(v))^2  & \leq \sum_{k=1}^K I^k \sum_{h\in \mathcal{H}_i} \mathbb{V}(P_{s_h^k,a_h^k},v) + 32\alpha\sum_{k=1}^K I^k \sum_{h\in \mathcal{H}_i}\sqrt{  (\phi_h^k)^{\top}(\bar{\Lambda}^k)^{-1}(\phi_h^k) } + 1.\end{align}

Using Lemma~\ref{lemma:variance_bound}, we have that: for any $v\in \mathcal{V}_i$
\begin{align}
\sum_{k=1}^K \sum_{h\in \mathcal{H}_i} \mathbb{V}(P_{s_h^k,a_h^k},v)\leq 36Kd^2\log(1/\epsilon)\iota, 
\end{align}
%which further implies that: with probability $1-10K^2H^2\delta$, for any $v\in {\mathcal{W}_{\epsilon}}$, 
%\begin{align}
%\sum_{k=1}^K I^k\sum_{h\in \mathcal{H}_i} (\sigma_h^k(v))^2\leq K\gamma_1,\label{eq:boundsigmah}
%\end{align}
Let $\gamma_1 =36d^2\log(1/\epsilon)\iota$. 
%By the successful event $\mathcal{G}$, we have that
%\begin{align}
%\sum_{k=1}^K I^k \sum_{h\in \mathcal{H}_i}\mathbb{V}(P_{s_h^k,a_h^k}, v) \leq \sum_{k=1}^K I^k\sum_{h\in \mathcal{H}_i} (\sigma_h^k(v))^2\leq K\gamma_1.\nonumber
%\end{align}
Noting that $$\mathbb{V}(P_{s_h^k,a_h^k},v) = (\phi_h^k)^{\top} \mu^{\top}v^2- \left((\phi_h^k)^{\top} \mu^{\top}v)\right)^2 = (\phi_h^k)^{\top} \theta(v^2) - ((\phi_h^k)^{\top}\theta(v))^2,$$ we have that for any $v\in \mathcal{V}_i$,
\begin{align}
\sum_{k=1}^K  \sum_{h\in \mathcal{H}_i} \left((\phi_h^k)^{\top} \theta(v^2) - ((\phi_h^k)^{\top}\theta(v))^2  \right) =\sum_{k=1}^K \sum_{h\in \mathcal{H}_i} \mathbb{V}(P_{s_h^k,a_h^k},v) \leq K\gamma_1.\label{eq:2822}
\end{align}

 By regarding 
$\left[\begin{array}{cc}
  \phi_h^k(\phi_h^k)^{\top}   &  \phi_h^k\\
   (\phi_h^k)^{\top} & 1
\end{array}\right]_{k\in [K],h \in 
\mathcal{H}_i}$ and $\left[\begin{array}{cc}
    {-}\theta(v)\theta(v)^{\top} &  \frac{1}{2}\theta(v^2)\\
   \frac{1}{2}\theta(v^2)^{\top} & 0
\end{array}\right]_{v\in \mathcal{V}_i}$ as two groups of vectors with dimension $(d+1)^2$ and applying Lemma~\ref{lemma:tool5},  
we obtain that
 {
\begin{align}
&\sum_{k=1}^K\sum_{h\in \mathcal{H}_i}\max_{v\in \mathcal{V}_i}\mathbb{V}(P_{s_h^k,a_h^k},v)\nonumber
\\ & =\sum_{k=1}^K  \sum_{h\in \mathcal{H}_i }\max_{v\in \mathcal{V}_i}   \left((\phi_h^k)^{\top} \theta(v^2) - ((\phi_h^k)^{\top}\theta(v))^2  \right) \nonumber
\\ & \leq 2(d+1)^2 \max_{v\in \mathcal{V}_{i}}\sum_{k=1}^K\sum_{h\in \mathcal{H}_i} \left((\phi_h^k)^{\top} \theta(v^2) - ((\phi_h^k)^{\top}\theta(v))^2  \right)\nonumber
\\ & \leq 2(d+1)^{2}K\gamma_1 .\label{eq:2823}
\end{align}
}
By \eqref{eq:ff4} and \eqref{eq:2823}, we obtain that 
%Using Lemma~\ref{lemma:tool1} and \eqref{eq:2823}, we further have that
%\begin{align}
% &\sum_{k=1}^K\sum_{h\in \mathcal{H}_i}\max_{v\in \mathcal{V}_i}\mathbb{V}(P_{s_h^k,a_h^k},v) \sum_{k=1}^K  \sum_{h\in \mathcal{H}_i}\max_{v\in \mathcal{V}_i}(\sigma_h^k(v))^2 \nonumber
% \\ & \leq \sum_{k=1}^K I^k \sum_{h\in \mathcal{H}_i}\left((\phi_h^k)^{\top} \theta(v^2) - ((\phi_h^k)^{\top}\theta(v))^2  \right) + 32\alpha \sum_{k=1}^K I^k \sum_{h\in \mathcal{H}_i}\sqrt{(\phi_h^k)^{\top} (\bar{\Lambda}^k)^{-1}\phi_h^k } + 2 \nonumber
% \\ & \leq 2(d+1)^2K\gamma_1 + 32\alpha \sum_{k=1}^K \min\left\{ \sum_{h\in \mathcal{H}_i}\beta_h^k,1 \right\} + 2.\label{eq:2824}
%\end{align}

%Combining \eqref{eq:betabound3} and \eqref{eq:2824}, we learn that 
%\begin{align}
%\sum_{k=1}^K \min\left\{ \sum_{h\in \mathcal{H}_i}\beta_h^k,1 \right\}\leq \sqrt{\gamma}\sqrt{ 2(d+1)^2K\gamma_1 + 32\alpha \sum_{k=1}^K \min\left\{ \sum_{h\in \mathcal{H}_i}\beta_h^k,1 \right\} + 2} +\gamma, 
%\end{align}
%which implies that
\begin{align}
 {\sum_{k=1}^K \min\left\{ \sum_{h\in \mathcal{H}_i}\beta_h^k,1 \right\} = \tilde{O}(\gamma \sqrt{d^2 \gamma_1 K})=\tilde{O}(d^3\sqrt{K})}.\label{eq:lll3}
\end{align}

Taking sum over $i$, putting \eqref{eq:mark1}, \eqref{eq:boundr}, \eqref{eq:exp2} and \eqref{eq:lll3} together, and noting that $\epsilon =\frac{1}{K^4H^4}$, with probability $1-10KH\delta$, it holds that
\begin{align}
 \sum_{k=1}^K T_1(k) \leq \tilde{O}(d^{5.5}\sqrt{K}+d^4\sqrt{ K}+d^{6.5})  = \tilde{O}(d^{5.5}\sqrt{K}+d^{6.5}) .\label{eq:boundt2}
\end{align}

\subsubsection{Bound of Term $T_2(k)+T_3(k)$ (Proof of Lemma~\ref{lemma:bdt23})}\label{sec:pft23}

With a slight abuse of notation, here we use $pv$ as shorthand of $p^{\top}v$ for $p\in \Delta^{S}$ and $v\in \mathbb{R}^S$.  

Recall that 
\begin{align}
T_2(k):&  =   \min \left\{\sum_{h=1}^H \left((\phi_h^k)^{\top}\mu^{\top}V_{h+1}^k - V_{h+1}^k(s_{h+1}^k) \right) ,1\right\}  \nonumber
\\ & = \min \left\{ \sum_{h=1}^H \left((P_{s_h^k,a_h^k}V_{h+1}^k - V_{h+1}^k(s_{h+1}^k) \right) ,1   \right\}.
\end{align}
Using Lemma~\ref{lemma:freeman_zhang}, with probability $1-4Kh\delta$, it holds that
\begin{align}
\sum_{k=1}^K T_2(k) & \leq 2\sqrt{2}\sqrt{ \sum_{k=1}^K \sum_{h=1}^H \mathbb{V}(P_{s_h^k,a_h^k},V_{h+1}^k)\iota } + 3\iota \nonumber
\\ & \leq 2\sqrt{2}\sqrt{ \sum_{k=1}^K \sum_{h=1}^H \mathbb{V}(P_{s_h^k,a_h^k},V_{h+1}^k)\iota } + 3\iota. \nonumber
\end{align}

Using Lemma~\ref{lemma:freeman_zhang} and Lemma~\ref{lemma:add1}, with probability $1-4KH\delta$,
\begin{align}
\sum_{k=1}^K \sum_{h=1}^H \mathbb{V}(P_{s_h^k,a_h^k},V_{h+1}^k) &  = \sum_{k=1}^K \sum_{h=1}^H ( P_{s_h^k,a_h^k} (V_{h+1}^k)^2  - ( P_{s_h^k,a_h^k} V_{h+1}^k)^2  ) \nonumber
\\ & \leq \sum_{k=1}^K \sum_{h=1}^H ( (V_h^k)^2(s_h^k)- (P_{s_h^k,a_h^k}V_{h+1}^k)^2   ) + \sum_{k=1}^K \sum_{h=1}^H (P_{s_h^k,a_h^k}-\textbf{1}_{s_{h+1}^k} (V_{h+1}^k)^2  \nonumber
\\ & \leq 2K + 4\sqrt{2}\sqrt{ \sum_{k=1}^K \sum_{h=1}^H \mathbb{V}(P_{s_h^k,a_h^k},V_{h+1}^k)\iota } + 3\iota.
\end{align}
Solving the equation above, we have that $\sum_{k=1}^K T_2(k)\leq  \sqrt{16K + 240\iota}+3\iota$.

Noting that $\mathbb{E}[\sum_{h=1}^H r_h^k]=V_1^{\pi^k}(s_1^k)$ and $0\leq \sum_{h=1}^H r_h^k \leq 1$ for any $k\in [K]$, with probability $1-2\delta$, it holds that   $\sum_{k=1}^K T_3(k)\leq 2\sqrt{2K\iota}+2\iota$. The proof is completed.

\subsubsection{Additional Lemmas}

\begin{lemma}\label{lemma:tool1}
Fix $v\in \mathbb{R}^{\mathcal{S}}$ such that $\|v\|_{\infty}\leq  1$. Define $\theta(v)=\mu_{P}^{\top}v$ and $\theta(v^2)=\mu_{P}^{\top}v^2$. Fix $k\in [K]$. Let $(\theta,\tilde{\theta},\Lambda)$ be the output of Algorithm~\ref{alg:vlambdae} with input as $\{s_{h'}^{k'},a_{h'}^{k'},s_{h'+1}^{k'})\}_{h'\in[H],k'\in [k-1]}$ and $v$.
Recall
 \begin{align}
\kappa   &  = 13\sqrt{6d^2\log^2(KH/\delta)} +72\log(KH/\delta)\leq \alpha\nonumber
\end{align}
For any $0<\delta<1$, with probability $1-10KH\delta/|\mathcal{W}_{\epsilon}|$, it holds that
\begin{align}
 \|\theta(v)-\theta\|_{\Lambda}\leq \kappa \mathrm{\quad and \quad } \|\theta(v^2) - \tilde{\theta}\|_{\Lambda}\leq 4\kappa.\label{eq:xxx21}
\end{align}
\end{lemma}
\begin{proof}[Proof of Lemma~\ref{lemma:tool1}] 
For convenience, we regard the sample $\{s_{h'}^{k'},a_{h'}^{k'},s_{h'+1}^{k'}\}$ as the $H(k'-1)+h$-th sample and rewrite $x_{h'}^{k'}$ as $x_{H(k'-1)+h}$ where $x$ can be any proper notations. 

Recall that $\mathbb{V}(P_{s,a},v)$ denote the variance of $ v(s')$ where $s'$ is the reward function and next state by taking $(s,a)$. Let $\mathrm{var}_i(v)$ be shorthand of $\mathbb{V}(P_{s,a},v)$. 
Let $\bar{\mathrm{var}}_i(v)$ denote the variance of $v^2(s')$ by taking state-action $(s_i,a_i)$.
By Lemma~\ref{lemma:add1}, we have that $\bar{\mathrm{var}}_i(v)\leq 4\mathrm{var}_i(v)$.

Let $\{\sigma_{i},\Lambda_i, \tilde{b}_i, b_i\}_{i\geq 1}$ be the variables computed in Algorithm~\ref{alg:vlambdae}. 
 We then have the following claim.

 \begin{claim} \label{claim:1}If it holds that
\begin{align}
\sigma^2_{i}\geq \mathrm{var}_{i}(v) +2\alpha\|\phi_i\|_{\Lambda^{-1}_{i-1}}\label{eq:tar1}
\end{align}
for each $i\leq (k-1)H$, then \eqref{eq:xxx21} holds with probability $1-5\delta/|\mathcal{W}_{\epsilon}|$.
\end{claim}
\begin{proof} [Proof of Claim~\ref{claim:1}]
 Fix $1\leq i \leq (k-1)H$. Recall the definition of $\Lambda_{i-1}, b_{i-1}, \theta_{i-1},\tilde{b}_{i-1}$ and $\tilde{\theta}_{i-1}$ in Algorithm~\ref{alg:vlambdae}. 
Because $\bar{\mathrm{var}}_i(v)\leq 4\mathrm{var}_i(v)$, we have that
\begin{align}
4\sigma_i^2 \geq\bar{\mathrm{var}}_i(v)+ 8\alpha\|\phi_i\|_{\Lambda^{-1}_{i-1}} \label{eq:kp1} 
\end{align}
for each $i\leq (k-1)H$.  Let $\epsilon_i=  v(s_{i+1})- \mathbb{E}[v(s_{i+1})|\mathcal{F}_i]$. 
Using Lemma~\ref{lemma:take1} with $\{\psi_i, \zeta_i \}_{i\geq 1}$ as
$\left\{\frac{\phi_i}{\sigma_i}, \frac{ \epsilon_i }{\sigma_i} \right\}_{i=1}^{(k-1)H}$
 and parameters as $\sigma^2 = 1,  R=H^2, L=H^2, c=1/(KH)^3, \lambda = 1/H^2$, with probability $1-5\delta/|\mathcal{W}_{\epsilon}|$
\begin{align}
 & \| \theta_{i-1} - \theta(v)\|_{\Lambda_{i-1}} \nonumber
 \\ & \leq  12\sqrt{6d^2\log^2(KH/\delta)} +72d\log(KH/\delta)\max_{1\leq i'\leq i} |\epsilon_{i'}/\sigma_{i'}|\cdot \min\{1, \|\phi_{i'}/\sigma_{i'}\|_{\Lambda^{-1}_{i'-1}}\} + \frac{1}{H}\nonumber
 \\ & \leq   12\sqrt{6d^2\log^2(KH/\delta)} +72d\log(KH/\delta)\max_{1\leq i'\leq i} \frac{|\epsilon_{i'}| \|\phi_{i'}\|_{\Lambda_{i'-1}^{-1}}}{ \sigma^2_{i'}} + \frac{1}{H}\nonumber
 \\ & \leq 12\sqrt{6d^2\log^2(KH/\delta)} +72d\frac{\log(KH/\delta)}{\alpha} + \frac{1}{H} \label{eq:xxx10}
 \nonumber\\ & \leq  \kappa.\nonumber
\end{align}
Here the second last inequality is by \eqref{eq:tar1}.

Let $\bar{\epsilon}_i =  v^2(s_{i+1})- \mathbb{E}[v^2(s_{i+1})|\mathcal{F}_i] $.
By setting  $\{\psi_i, \zeta_i \}_{i\geq 1}$ as  $\left\{\frac{\phi_i}{2\sigma_i}, \frac{ \bar{\epsilon}_i}{2\sigma_i} \right\}_{i=1}^{(k-1)H}$, by \eqref{eq:kp1},  with probabbility $1-5\delta/|\mathcal{\epsilon}|$,
\begin{align}
 & \|\tilde{\theta}_{i-1}-\theta(v^2)\|_{\Lambda_{i-1}/4} \nonumber
 \\ & \leq  \left( 12\sqrt{12d^2\log^2(KH/\delta)} +72d\log(KH/\delta)\max_{1\leq i'\leq i} |\bar{\epsilon}_{i'}/(2\sigma_{i'})|\cdot \min\{1, \|\phi_{i'}/(2\sigma_{i'})\|_{4\Lambda^{-1}_{i'-1}}\} + \frac{1}{H}\right) \nonumber
 \\ & \leq   \left( 12\sqrt{12d^2\log^2(KH/\delta)} +72d\log(KH/\delta)\max_{1\leq i'\leq i} \frac{ 2\|\phi_{i'}\|_{\Lambda_{i'-1}^{-1}} }{4\sigma_{i'}^2} + \frac{1}{H}\right) \nonumber
 \\ & \leq \left( 12\sqrt{12d^2\log^2(KH/\delta)} +72d\log(KH/\delta)/(2\alpha) + \frac{1}{H}\right) \nonumber
 \\ & \leq 2\kappa,
 \end{align}
which implies that $\|\tilde{\theta}_{i-1}-\theta(v^2)\|_{\Lambda_{i-1}}\leq 4\kappa$. 

% {The proof is finished}

 % Noting that $\sigma^2_i \geq \min\{1, 2\|\phi_i\|_{\Lambda_{i-1}} \}\leq  \kappa$, 
% the conclusion holds easily . 
\end{proof}

So it suffices to prove that \eqref{eq:tar1} holds for any $1\leq i \leq (k-1)H$.  
To prove \eqref{eq:tar1}, we use induction on $i=1,2,\ldots, n = (k-1)H$. By the update rule in Algorithm~\ref{alg:vlambdae}, we have that
\begin{align}
\sigma^2_{i}= \phi_i^{\top}\tilde{b}_i- (\phi_i^{\top}b_i)^2 + 16\alpha \sqrt{\phi_i^{\top} (\Lambda_{i-1})^{-1}\phi_i} +4\epsilon.\label{eq:redefs}
\end{align}
By assuming $\sigma^2_{i'}\geq \mathrm{var}_{i'}(v) +2\alpha\|\phi_{i'}\|_{\Lambda^{-1}_{i'-1}}$ holds for $1\leq i'\leq i-1$, using Claim~\ref{claim:1} with samples as $\{s_{i'},a_{i'},s_{i'+1}, v(s_{i'+1})\}_{i'=1}^{i-1}$,  with probability $1-5\delta/|\mathcal{W}_{\epsilon}|$ it holds that
\begin{align}
|\phi_i^{\top}\theta_{i-1}-\phi_i^{\top}\theta(v) |\leq \|\phi_i\|_{\Lambda_{i-1}^{-1}} \cdot \|\theta_{i-1}-\theta(v)\|_{\Lambda_{i-1}} \leq \kappa \sqrt{\phi_i^{\top}\Lambda_{i-1}^{-1}\phi_i} \leq \alpha \sqrt{\phi_i^{\top}\Lambda_{i-1}^{-1}\phi_i} .\label{eq:261}
\end{align}

Note that $\theta(v^2) = \mu_{P}^{\top}v^2$. Using Claim~\ref{claim:1} again, with probability $1-5\delta/|\mathcal{W}_{\epsilon}|$,
\begin{align}
|\phi_i^{\top}\tilde{\theta}_{i-1} - \phi_i^{\top} \theta(v^2)| \leq 4\alpha \sqrt{ \phi_i^{\top} (\Lambda_{i-1})^{-1}\phi_i}. \label{eq:262}
\end{align}

By noting that $\mathrm{var}_i(v) = \phi_i^{\top}\tilde{\theta}(v)-(\phi_i^{\top}\theta(v))^2$, we have that
\begin{align}
\sigma_i^2 \geq  \mathrm{var}_i(v) + 10\alpha \sqrt{\phi_i^{\top}(\Lambda_{i-1})^{-1}\phi_i}.\nonumber
\end{align}
Also recalling that $\bar{\mathrm{var}}_i(v)\leq 4\mathrm{var}_i(v)$, 
 we have that
\begin{align}
4\sigma_i^2 \geq \bar{\mathrm{var}}_i(v) + 8\alpha\|\phi_i\|_{\Lambda_{i-1}^{-1}} .
\end{align}

The proof is completed.
\end{proof}

\begin{lemma}\label{lemma:variance_bound} 
 With probability $1-4K^2H^2\delta$, for any $i\in [i_{\mathrm{max}}]$ and $v\in \mathcal{V}_i$, it holds that
 \begin{align}
\sum_{k=1}^K  \sum_{h\in \mathcal{H}_i}\mathbb{V}(P_{s_h^k,a_h^k},v)\leq  K(36\iota+18d+10\log(KH))\nonumber
 \end{align}

\end{lemma}

\begin{proof} With a slight abuse of notation, here we use $pv$ as shorthand of $p^{\top}v$ for $p\in \Delta^{S}$ and $v\in \mathbb{R}^S$.  
Define $\delta' = \delta/|\mathcal{W}_{\epsilon}|$ and let $\iota' = \log(2/\delta')$.  Fix $v\in  {\mathcal{W}_{\epsilon}}$ and let   {$\bar{v}(s)=\max\{\max_{a}P_{s,a}v,v(s) \}$}.  Fix $1\leq h_1\leq h_2\leq H+1$. With probability $1-\delta'$, it holds that
 \begin{align}
\sum_{h=h_1}^{h_2} \mathbb{V}(P_{s_h^k,a_h^k}, v)&  = \sum_{h=h_1}^{h_2}\left( P_{s_h^k,a_h^k} v^2 - (P_{s_h^k,a_h^k}v)^2 \right)  \nonumber
\\ & =\sum_{h=h_1}^{h_2}\left( P_{s_h^k,a_h^k}v^2 - v^2(s_{h+1}^k)  \right) + \sum_{h=h_1}^{h_2} \left(  v^2(s^k_{h+1})  -  (P_{s_h^k,a_h^k} v)^2   \right)\nonumber
\\ & \underset{(a)}{\leq } 2 \sqrt{   \sum_{h=h_1}^{h_2} \mathbb{V}(P_{s_h^k,a_h^k},v)     \iota'} + \sum_{h=h_1}^{h_2}\left( v^2(s_h^k) - (P_{s_h^k,a_h^k}v)^2 \right) + 4\iota'+2 \nonumber
\\ & \underset{(b)}{\leq} 4\sqrt{\sum_{h=h_1}^{h_2} \mathbb{V}(P_{s_h^k,a_h^k},v)     \iota'} + 2\sum_{h=h_1}^{h_2}\max\{v(s_h^k) - P_{s_h^k,a_h^k}v, 0  \} +4 \iota'+2 \nonumber
\\ & \underset{(c)}{\leq} 4\sqrt{\sum_{h=h_1}^{h_2} \mathbb{V}(P_{s_h^k,a_h^k},v)     \iota'} + 2\sum_{h=h_1}^{h_2} \left( \bar{v}(s_h^k) - P_{s_h^k,a_h^k}v\right) + 4\iota'+2\nonumber
\\ &  \underset{(d)}{\leq} 4\sqrt{\sum_{h=h_1}^{h_2} \mathbb{V}(P_{s_h^k,a_h^k},v)     \iota'} + 2\sum_{h=h_1}^{h_2} \left( v(s_h^k) - P_{s_h^k,a_h^k}v\right) + 4\iota' +  {2(h_2-h_1+1)\|\bar{v}-v\|_{\infty}}+2\nonumber
\\ & \underset{(e)}{\leq} 4\sqrt{\sum_{h=h_1}^{h_2} \mathbb{V}(P_{s_h^k,a_h^k},v)     \iota'} + 4\sqrt{\sum_{h=h_1}^{h_2} \mathbb{V}(P_{s_h^k,a_h^k}, v) } + 12\iota' +2(h_2 - h_1+1)\|\bar{v}-v\|_{\infty}+2\nonumber
\\ &  = 8\sqrt{ \sum_{h=h_1}^{h_2}\mathbb{V}(P_{s_h^k,a_h^k},v)     \iota' } +12\iota'+2(h_2 - h_1+1)\|\bar{v}-v\|_{\infty}+2,\nonumber
 \end{align}
which further implies that
\begin{align}
\sum_{h=1}^H \mathbb{V}(P_{s_h^k,a_h^k},  {v}) \leq 36\iota'+6(h_2-h_1+1)\|\bar{v}-v\|_{\infty}+6.
\end{align}
Here $(a)$ and $(e)$ hold by Lemma~\ref{lemma:tool2}, $(b)$ holds by Lemma~\ref{lemma:add1},   $(c)$ holds by the fact that $\bar{v}(s)\geq P_{s,a}v$ for any proper $(s,a)$,  {and $(d)$ holds because $(\bar{v}(s_h^k)-v(s_h^k))\leq \|\bar{v}-v\|_{\infty}$ for all proper $(h,k)$}.

Now we bound  {$\|\bar{v}-v\|_{\infty}$}. By definition of $\mathcal{V}_i$, if $v\in \mathcal{V}_i$, there then exists some $k\in [K],h\in \mathcal{H}_i$, such that $\|v-V_{h+1}^k\|_{\infty}\leq \epsilon$. It then follows that $\bar{v}(s)\leq \max_{a}P_{s,a} V_{h+1}^k  +\epsilon\leq V_h(s)+\epsilon$. Therefore, by Lemma~\ref{lemma:tool3}, we have that $\|\bar{v}-v\|_{\infty}\leq \epsilon +\|V_h^k-V_{h+1}^k\|_{\infty}\leq \epsilon + \frac{2^i}{H}$.  By choosing $[h_1,h_2]=\mathcal{H}_i$, we have that with probability $1-\delta'$
 {
\begin{align}
\mathbb{I}[v\in \mathcal{V}_i]\sum_{h=h_1}^{h_2} \mathbb{V}(P_{s_h^k,a_h^k}, v)\leq  8\sqrt{ \sum_{h=h_1}^{h_2}\mathbb{V}(P_{s_h^k,a_h^k},v)     \iota' } +12\iota'+ 2H\epsilon +4,\nonumber
\end{align}
which implies}
\begin{align}
 {\mathbb{I}[v\in \mathcal{V}_i]}\sum_{h=1}^H \mathbb{V}(P_{s_h^k,a_h^k},  {v}) \leq 36\iota'+12+6H\epsilon
\end{align}
 {with probability $1-\delta'$.} With a union bound over $\mathcal{W}_{\epsilon}$ and $i\in [i_{\mathrm{max}}]$, we have that: with probability $1-4K^2H^2\delta$,
 \begin{align}
\sum_{k=1}^K  \sum_{h\in \mathcal{H}_i}\mathbb{V}(P_{s_h^k,a_h^k},v)\leq  K(36\iota+18d+10\log(KH))\nonumber
 \end{align}
for any $i$ and $v\in \mathcal{V}_i$.
The proof is completed.
\end{proof}

\begin{lemma}\label{lemma:tool3}
For any $k\in [K], h\in [H]$, it holds that $ {\|V_h^k-V_{h+1}^k\|_{\infty}}\leq \frac{2d}{H-h+1}$.
\end{lemma}
\begin{proof} Fix $k$. Let $l_h =\|V_h^k -V_{h+1}^k\|_{\infty}$. 
Let $\Gamma$ denote the Bellman operator under transition kernel $\mu^k$. Since $\|\Gamma (v_1-v_2)\|_{\infty}\leq  \|v_1-v_2\|_{\infty}$ for any $v_1,v_2\in \mathbb{R}^S$, $l_h$ in non-decreasing in $h$. So it suffices to bound $\sum_{h=1}^H l_h$.

By Lemma~\ref{lemma:dual}, for any $s$,
\begin{align}
 & V_h(s)-V_{h+1}(s)\nonumber
 \\&\leq \max_{\phi\in \Phi}  \phi^{\top}(\mu^k)^{\top}(V_{h+1}-V_{h+2}) \nonumber
\\ & \leq 2d\bar{\phi}^{\top}(\mu^k)^{\top}(V_{h+1}-V_{h+2}).
\end{align}
It then follows that $\sum_{h=1}^H l_h \leq 2d\bar{\phi}^{\top} (\mu^k)^{\top}\sum_{h=1}^H (V_{h+1}-V_{h+2})\leq 2d$.
The proof is completed.

\end{proof}

\section{Discussion about VOFUL}\label{app:voful}

We first introduce the VOFUL estimator in \cite{zhang2021variance}. The Algorithm is presented in Algorithm~\ref{alg:voful}. Then we have the lemma to bound the error due to uncertainty of reward parameter.

%With a slight abuse of notations, we use $\phi_{(k-1)H+h}$ and $r_{(k-1)H+h}$ to denote $\phi_h^k$ and $r_h^k$   respectively.

 %We will use this lemma to bound the regret due to the confidence region $\{\Theta^{k}\}_{k=1}^K$ (see \eqref{eq:boundr} in Section~\ref{sec:reg}).
%\simon{we need to give some discussions on this lemma. How will we apply it? Or we can also put it in the appendix, since this is not our main contributions anyway.}

\begin{lemma}\label{assum1} Let $\{\phi_i\}_{i=1}^n$ be a group of feature vectors in $\mathbb{R}^d$ such that $\|\phi_i\|_2\leq 1$, and $r_i = \phi_i^{\top}\theta^*+\epsilon_i\in [-1,1]$ for some $\theta^*\in \mathbb{R}^d$ with $\|\theta^*\|_2\leq 1$. Let $\bar{\mathcal{F}}_i$ be the $\sigma$-field of $\{\phi_{i'},r_{i'}\}_{i'\leq i}$. Assume that $\mathbb{E}[\epsilon_{i}|\bar{\mathcal{F}}_{i-1}]=0$ and $\mathbb{E}[\epsilon_i^2 |\bar{\mathcal{F}}_{i-1}]=\sigma^2_i$ for any $i\geq 1$. 
Let $\{\Theta_i\}_{i=1}^n$ to be the confidence region 
for the true parameter $\theta^*$  in Line~\ref{line:confball} Algorithm~\ref{alg:voful}  with input as $\{ \phi_i   \}_{i=1}^n$ and $\{r_i\}_{i=1}^n$. 
It then holds that with probability $1-10 n\delta$:  
(\romannumeral1) $\theta_r\in \Theta_i$ for any $i\in [n]$; (\romannumeral2) For any sequence $0= i_1 <i_2<\dots < i_z = n$, it holds that
$\sum_{l=1}^{z-1}\min\left\{\sum_{i=i_l+1}^{i_{l+1}}(\max_{\theta\in \Theta_{i_l+1}}\phi_i^{\top}(\theta - \theta^*) ,1 \right\}  = \tilde{O}(d^{5.5}\sqrt{\sum_{i=1}^n \sigma_i^2}+d^{6.5}). $
\end{lemma}
\begin{proof}
Now we proof Lemma~\ref{assum1}. Firstly, (\romannumeral1) holds by the lemma below.
\begin{lemma}[Lemma 18 in \cite{zhang2021variance}]
With probability $1-3\log(n)\delta, \theta^* \in \Theta_i$ for any $i\in [n]$.
\end{lemma}

To verify (\romannumeral2), we define 
\begin{align}
\mathcal{L}:=\left\{l\in [z]|\exists j\in [L_2], u\in \mathcal{B'}, \sum_{i=1}^{i_{l+1}}\mathrm{clip}^2_j(\phi_i^{\top}u)+l_j^2\geq 4(d+2)^2 \sum_{i=1}^{i_{l}}\mathrm{clip}^2_j(\phi_i^{\top}u)+l_j^2  \right\}.\nonumber
\end{align}

With Lemma~\ref{lemma:ax1} below, we have that $|\mathcal{L}|\leq O(d\log^3(n))$.

\begin{lemma}\label{lemma:ax1}[Lemma 14 in \cite{zhang2021variance}] Let $f$ be a convex function.
Let $\phi_{1},\phi_{2},\ldots,\phi_{t} \in \mathcal{B}$ be a sequence of vectors. If there exists a sequence $0 = \tau_0<\tau_1<\tau_2<\ldots<\tau_z = t$ such that for each $1\leq  \zeta \leq z$, there exists $\mu_{\zeta}\in \mathcal{B}$ such that
	\begin{align}
	\sum_{i=1}^{\tau_{\zeta}}f(\phi_{i}\mu_{\zeta})+\ell^2 >4(d+2)^2\times \left( \sum_{i=1}^{\tau_{\zeta-1}}f(\phi_{i}\mu_{\zeta}) +\ell^2 \right),\label{eq:mul1}
	\end{align}
	then $z\leq O(d\log^2(dt/\ell))$.
\end{lemma}

Then we have that
\begin{align}
 & \sum_{l=1}^{z-1}\min\left\{\sum_{i=i_l+1}^{i_{l+1}}(\max_{\theta\in \Theta_{i_l+1}}\phi_i^{\top}(\theta - \theta^*) ,1 \right\} \nonumber
 \\ & \leq   \sum_{l\notin\mathcal{L}}\sum_{i=i_{l}+1}^{i_{l+1}}  \max_{\theta\in \Theta_{i_{l}+1}} (\phi_i)^{\top} (\theta-\theta^*) + O(d\log^3(n))\nonumber
 \\ & \leq  (d+2) \sum_{l\notin\mathcal{L}}\sum_{i=i_{l}+1}^{i_{l+1}}  \max_{\theta\in \Theta_{i_{l+1}+1}} (\phi_i)^{\top} (\theta-\theta^*) + O(d\log^3(n))
 \label{eq:exp4}
\\ & \leq (d+2) \sum_{i=1}^{n}  \max_{\theta\in \Theta_{i}} (\phi_i)^{\top} (\theta-\theta^*) + O(d\log^3(n))
 \\ & \leq \tilde{O}\left( d^{5.5}\sqrt{\sum_{i=1}^n\sigma_i^2 } +d^6 \right).\nonumber
\end{align}
Here \eqref{eq:exp4} is by the definition of $\mathcal{L}$, and the last inequality holds because Lemma 19 in \cite{zhang2021variance}, which states that
\begin{lemma}\label{lemma:rebound}[Lemma 19 in \cite{zhang2021variance}] Recall the definition of $\{\Theta_i\}_{i=1}^n$ in Algorithm~\ref{alg:voful}. 
With probability $1-5\log(n)\delta$, it holds that 
\begin{align}
\sum_{i=1}^n  \max_{\theta\in \Theta_i}\phi_i^{\top} (\theta - \theta^*) = \tilde{O}\left(d^{4.5}\sqrt{\sum_{i=1}^{n}\sigma_i^2}+d^5\right).\nonumber
\end{align}
\end{lemma}

The proof is completed.

\end{proof}
\setlength{\textfloatsep}{0.1cm}
\setlength{\floatsep}{0.1cm}
\begin{algorithm}[t]
	\caption{VOFUL: \textbf{V}ariance-Aware \textbf{O}ptimism in the \textbf{F}ace of \textbf{U}ncertainty for \textbf{L}inear Bandits}
	\begin{algorithmic}[1] \label{alg:voful}
 \STATE{\textbf{Input :} $\{\phi_k\}_{k=1}^n$, $\{r_i\}_{k=1}^n$}
		\STATE { \hspace{0ex}\textbf{Initialize:}  $L_2 = \lceil \log_2 n \rceil$,  $\ell_{j} = 2^{2-j} \forall 1\leq j \leq L_2+1, \iota = 16d \ln \frac{dn}{\delta},\Lambda_2 = \{1,2,\ldots,L_2 + 1\}$, $\Theta_1 =\mathcal{B}(2), $ Let $\mathcal{B}'$ be an $n^{-3}$-net of $\mathcal{B}$ with size not larger than $(\frac{4}{n})^{3d}$ \label{line:bandit_init}}
		\STATE{\textbf{Construct Confidence Set:}  
			\STATE For each $\theta \in \mathcal{B}$, define $\epsilon_k(\theta) = r_k - (\phi_k)^{\top} \theta, \eta_k(\theta) = (\epsilon_k(\theta))^2$. \label{line:bandit_var_est}}
		\STATE{Define confidence set $\Theta_{k+1} = \left(\bigcap_{j \in \Lambda_2}\Theta^{j}_{k+1}\right)\cap \Theta_{k},$ where
%		\vspace{-1em}
			\begin{align}
			\hspace{-2em}\Theta^{j}_{k+1} = \bigg\{ \theta \in \mathcal{B} ~ &:  \abs{\sum_{v = 1}^{k} \mathrm{clip}_j((\phi_{v})^{\top} \mu) \epsilon_{v}(\theta)}  \leq \sqrt{ \sum_{v = 1}^{k} \mathrm{clip}_j^2((\phi_{v})^{\top} u)\eta_{v}(\theta) \iota} + \ell_{j} \iota, \forall \mu \in \mathcal{B}'\bigg\}
  \end{align} 
			and $\mathrm{clip}_j(\cdot) = \mathrm{clip}(\cdot, \ell_j) $, $\mathrm{clip}(u,l) = \min\{|u|,l\}\frac{u}{|u|}$. \label{line:confball}
}
	\end{algorithmic}
\end{algorithm}

\end{document}